\documentclass[a4paper,12pt]{article}
\usepackage{amsmath, amsthm, amssymb} 
\usepackage{centernot}
\usepackage[right]{eurosym}
\usepackage{tikz}
\usepackage{url}
\usepackage[english]{babel}
\usepackage[utf8]{luainputenc}
\usepackage{tikz}
\usepackage{centernot}
\usepackage[shortlabels]{enumitem}
\usepackage{ifthen}
\usepackage{verbatim}   
\usepackage{multicol}

\numberwithin{equation}{section}

\newlist{choices}{enumerate}{1}
\setlist[choices]{label*=(\alph*)}

\SetEnumitemKey{twocol}{
  before=\raggedcolumns\begin{multicols}{2},
  after=\end{multicols}}

\renewcommand{\subset}{\subseteq}
\renewcommand{\epsilon}{\varepsilon}
\newcommand{\Var}{\text{Var}}

\newcommand{\MSE}{\text{MSE}}

\newcommand{\cA}{\mathcal{A}}

\newcommand{\cF}{\mathcal{F}}

\newcommand{\cH}{\mathcal{H}}

\newcommand{\cM}{\mathcal{M}}

\newcommand{\cO}{\mathcal{O}}
\newcommand{\cR}{\mathcal{R}}

\newcommand{\cT}{\mathcal{T}}

\newcommand{\N}{\mathbb{N}}

\newtheorem{theo}{Theorem}[section]
\newtheorem{lemma}[theo]{Lemma}

\newtheorem{corollary}[theo]{Corollary}

\theoremstyle{definition}
\newtheorem{defi}[theo]{Definition}

\newtheorem*{bem*}{Comment}

\newtheorem{proposition}[theo]{Proposition} 
\newtheorem*{algorithm}{Algorithm}

\newcommand{\hidden}[1]{}

        {%
            \ifthenelse{\isundefined{\showtodos}}%
                    {\expandafter\comment}%
                    {}%
                    }%
         {%
            \ifthenelse{\isundefined{\showtodos}}%
                    {\expandafter\endcomment}%
                    {}%
          }

\title{Learning Optimal Search Strategies}

\author{Stefan Ankirchner \thanks{Stefan Ankirchner, Institute for Mathematics, University of Jena, Inselplatz 5, 07743 Jena, Germany. \emph{Email:} s.ankirchner@uni-jena.de.} 
\and Maximilian Philipp Thiel \thanks{Maximilian Philipp Thiel, Institute for Mathematics, University of Jena, Inselplatz 5, 07743 Jena, Germany. \emph{Email:} maximilian.thiel@uni-jena.de.}  }

\begin{document}

\maketitle

\begin{abstract}

We explore the question of how to learn an optimal search 
strategy within the example of a parking problem where parking opportunities arrive according to an unknown inhomogeneous Poisson process.
The optimal policy is a threshold-type stopping rule characterized by an indifference position. We propose an algorithm that learns this threshold by estimating the integrated jump intensity rather than the intensity function itself. We show that our algorithm achieves a logarithmic regret growth, uniformly over a broad class of environments. Moreover, we prove a logarithmic minimax regret lower bound, establishing the growth optimality of the proposed approach.  
\end{abstract}
\begin{center}\footnotesize
  \begin{tabular}{r@{ : }p{10cm}}
      {\itshape 2020 MSC} & Primary: 60G40, 93E35; secondary: 62L05, 68Q32.\\
      {\itshape Keywords} &  optimal stopping, parking problem, reinforcement learning, regret. 
  \end{tabular}
\end{center}

\section*{Introduction}
The parking problem is a classical example of a search problem. 
To describe a standard version of it, suppose that an agent is driving along a street, that she cannot make a U-turn and that she can only see whether the next lot is free, but not which of the following ones. If the agent arrives at a free lot, then she has to decide whether to take it or not. Once a free lot is discarded, it is discarded forever. A decision rule can be modeled as a stopping time. The parking problem consists of finding the stopping time that minimizes the expected distance of the taken lot to some given target. 

To determine an optimal stopping time, one needs to know the distribution of the position of free parking lots. But what if the agent does not know the distribution? 

The parking problem is also a paradigm of a problem that usually needs to be solved not once, but in many consecutive rounds, e.g.\ every morning when driving to your workplace. If the stopping agent does not know the distribution of the free lots, then she can learn it over time by observing in each round the positions of free lots up to the actual chosen one. 

In the present article we address the question of what constitutes a good strategy for choosing a parking lot round after round. We do so within the framework of a continuous-time model, where free parking lots arrive according to an inhomogeneous Poisson process. 
Within this model the optimal stopping rule is of threshold type: there is a position $b^*$ after which it is optimal to take the first free lot. The position $b^*$ can be characterized as an indifference level: if the lot at $b^*$ is free, then the agent is indifferent between taking it or taking the next free one. 

We assume that the stopping agent does not know the jump intensity of the Poisson process. We propose a specific algorithm, called indifference level updating (ILU), that  estimates the integrated jump intensity and updates the stopping rule at the beginning of each epoch, where the epochs have doubling length, based on the observations made so far. 

In order to assess the quality of our algorithm we compute the growth rate of the regret, i.e.\ the accumulated difference between the expected distance to the target and the minimal expected distance when choosing the optimal stopping rule. Indeed, we show that the regret implied by the ILU algorithm grows logarithmically, uniformly for a large class of distributions. 

A crucial property of the ILU algorithm is that it does not estimate the jump intensity function, but the \emph{integrated} jump intensity. For the latter we have estimators with a mean square error (MSE) converging to zero at the rate $1/n$, where $n$ is the number of independent process observations. We assume that the jump intensity is once continuously differentiable, which implies that the round-wise regret can be bounded against the MSE of the estimator. Hence an MSE in $\cO(1/n)$ entails a logarithmic growth of the accumulated regret. We remark that any estimator of the intensity function itself, e.g.\ a kernel estimator, converges at a slower rate than $\cO(1/n)$, and hence entails a regret growing faster than the logarithm.

To show that the ILU algorithm is good, we prove that the minimax regret grows logarithmically as well. This means that there is no algorithm that has, uniformly over all environments, a regret rate that grows slower than the logarithm. In this sense the ILU algorithm is a good method for choosing parking lots.

While we present our results in the context of the parking problem, the underlying methods are not specific to this application. They apply more broadly to a class of timing and search problems with stochastic opportunity arrivals. To simplify the exposition, we develop our approach using the parking problem as a canonical example; it should be possible to carry out extensions to other settings with minor and natural adaptations.

\subsubsection*{Comparison with the literature} 

There is an extensive literature on the parking problem formulated as a stopping problem, see e.g.\ \cite{MacQueen} for an early reference in discrete time. A version of the parking problem in continuous time has been considered in \cite{sakaguchi1982optimal}. We refer to Chapter 2.5 in \cite{ferguson2018optimal} for a recent overview on the variants of the parking problem that have been solved already. 

The ILU algorithm described in the present article can be seen as an example of a model-based reinforcement learning algorithm. Reinforcement learning (RL) is usually understood to comprise algorithms, e.g.\ the $q$-learning algorithm, that make few assumptions on the model and the system distribution. RL algorithms are, therefore, universally applicable. However, the algorithms can be quite inefficient for some applications.

To obtain efficient algorithms for learning solutions to stochastic control problems it is natural to use as much information on the model as possible. The ILU algorithm presented in this article is such an algorithm: it makes use of the fact that optimal stopping rules are of threshold type, i.e.\ there exists a position after which it is optimal to take the first free lot. Moreover, the algorithm exploits the fact that free lots arrive according to an inhomogeneous Poisson process. A crucial ingredient of the algorithm is the estimation of the integrated jump intensity of the Poisson process.  

\vspace{.2cm}
There is quite some literature on model-based RL algorithms for models where the state dynamics are described as a Markov decision process. We refer to \cite{m2023model} for a survey. 
Examples of model-based RL algorithms for continuous-time stochastic control problems are still rather scarce. A first approach for linear-quadratic problems in continuous time is presented in \cite{duncan2002adaptive}. The authors introduce a weighted least-squares algorithm for learning drift rates and hence the optimal control. The growth rate of the regret is not provided, and it seems that its determination has not yet been made. \cite{basei2022logarithmic,guo2021reinforcement,szpruch2021exploration} consider linear-quadratic problems in an episodic finite time horizon setting, providing non-asymptotic regret bounds. 

The article \cite{neuman2023statistical} considers propagator models and describes algorithms achieving sublinear regret. \cite{christensen2024data, christensen2023data, MR4746600, christensen21} examine singular and impulsive control problems in a non-parametric, ergodic setting, and obtain algorithms with sublinear regret rates of power type. The article \cite{ACKLBP} introduces a steering algorithm to keep a process as close to some given path as possible, exploiting that the process is driven by a Brownian motion and that the drift rates are from an unknown bounded interval.

\section{The parking problem in continuous time}
\label{optimal_criteria_chapter}

In this section we summarize some results on the parking problem in continuous time. 

Suppose that the parking spaces are located on the interval \([S,\infty)\), where \( S\in (-\infty,0) \). 
Assume that free parking lots arrive according to an inhomogeneous Poisson process with continuous jump intensity \(\lambda: [S, \infty) \to [c,\infty)\), where \(c\) is a positive constant. More precisely, let \((N_t^\lambda)_{t\ge S}\) a right-continuous process, defined on some probability space $(\Omega, \cF, P)$, such that the increments are independent and \((N_t^\lambda-N_s^\lambda)\) is Poisson distributed with the parameter \(\int_s^t \lambda(u) du\), for all \(s\le t\). We interpret the jump times of the process $(N_t)$ as the positions of the free parking lots the agent can use while driving along the street from $S$ to the right.   

In this section we assume that the jump intensity function $\lambda$ is known by the stopping agent. In the next section we will omit this assumption and assume that the stopping agent does not know $\lambda$, but can learn it by observing independent samples of the jump process $(N_t)$. 

We suppose that the agent wants to park her car as close as possible to the target $0$, in expectation. A policy for searching a parking lot can be described in terms of decision rule describing for each parking lot $t \in [S, \infty)$ whether to take it or not, in case it is free. 

It is straightforward to show that it is enough to consider only decision rules of the following threshold type: there is a threshold $b \in [S, 0]$ after which the first free lot is accepted; and before $b$ any free lot is discarded. Indeed, for any decision rule one can construct a threshold rule with an expected distance to the target that is not larger. 

We model the decision rule with threshold $b \in [S, 0]$ as the stopping time \(\tau_b := \inf\{t\ge b: N_t > N_b\}\). Observe that $\tau_b$ is the first jump time of $N$ after $b$; we interpret $\tau_b$ as the position of the first free parking space after $b$. 

Note that the aim to park the car as close as possible to the target amounts to the problem of finding the threshold $b$ for which $E|\tau_b|$ becomes minimal. We refer to $b^*$ such that $E|\tau_{b^*}| = \min_{b \in [S, 0]} E[|\tau_b|]$ as an optimal threshold. 

The next theorem provides a sufficient condition for $b^*$ to be an optimal threshold. 
\begin{theo}
\label{criteria}
Let \(b^* \in [S,0)\) be such that 
\begin{align}
\label{criteria_optimal}
\int_{b^*}^0 e^{\int_y^0 \lambda(u)du} \, dy =\int_0^\infty e^{-\int_0^y \lambda(u)du} \, dy.
\end{align}
Then \( b^*\) is an optimal threshold.
\end{theo}
\begin{proof}

First note that for all \(b \in [S,0]\) we have  
\[E|\tau_b|=-\int_b^0 y \lambda(y) e^{-\int_b^y \lambda(u)du} \, dy + \int_0^\infty y \lambda(y) e^{-\int_b^y \lambda(u)du} \, dy.\]
We first show that \eqref{criteria_optimal} is equivalent to the  FOC $\frac{\partial E|\tau_b|}{\partial b} = 0$. Afterwards we show that the FOC is a sufficient criterion for optimality. Indeed, 
\[\frac{\partial E|\tau_b|}{\partial b}=\lambda(b)\left[b-\int_b^0 y \lambda(y) e^{-\int_b^y \lambda(u)du} \, dy + \int_0^\infty y \lambda(y) e^{-\int_b^y \lambda(u)du} \, dy \right],\]
and hence $\frac{\partial E|\tau_b|}{\partial b} = 0$ is equivalent to 
\[b=\int_b^0 y \lambda(y) e^{-\int_b^y \lambda(u)du} \, dy - \int_0^\infty y \lambda(y) e^{-\int_b^y \lambda(u)du} \, dy.\]
By using \(\lambda(y) e^{-\int_b^y \lambda(u)du} = -\frac{\partial}{\partial y} e^{-\int_b^y \lambda(u) du}\) and applying integration by parts, the last equation can be rewritten as
\begin{align*}
b&=-\int_b^0 y  \frac{\partial}{\partial y} e^{-\int_b^y \lambda(u)du} \, dy + \int_0^\infty y \frac{\partial}{\partial y} e^{-\int_b^y \lambda(u)du} \, dy\\
&= -\left[y e^{-\int_b^y \lambda(u)du} \right]_b^0 + \int_b^0 e^{-\int_b^y \lambda(u)du} \, dy + \left[y e^{-\int_b^y \lambda(u)du} \right]_0^\infty - \int_0^\infty e^{-\int_b^y \lambda(u)du} \, dy\\
&=-\left(0-b\right)+ \int_b^0 e^{-(\int_b^0 \lambda(u)du-\int_y^0 \lambda(u)du)} \, dy + 0 - \int_0^\infty e^{-(\int_b^0 \lambda(u)du+\int_0^y \lambda(u)du)}\, dy\\
&=b + e^{-\int_b^0 \lambda(u)du} \left(\int_{b}^0 e^{\int_y^0 \lambda(u)du} \, dy -\int_0^\infty e^{-\int_0^y \lambda(u)du} \, dy\right).
\end{align*}
This implies that the FOC is equivalent to 
\[\int_{b}^0 e^{\int_y^0 \lambda(u)du} \, dy =\int_0^\infty e^{-\int_0^y \lambda(u)du} \, dy.\]

The previous calculations give the following representation of \(\frac{\partial E|\tau_b|}{\partial b}\):
\[\frac{\partial E|\tau_b|}{\partial b}=-\lambda(b)e^{-\int_b^0 \lambda(u)du} \left(\int_{b}^0 e^{\int_y^0 \lambda(u)du} \, dy -\int_0^\infty e^{-\int_0^y \lambda(u)du} \, dy\right).\]
We know that \(-\lambda(b)e^{-\int_b^0 \lambda(u)du}<0\) and \(\left(\int_{b}^0 e^{\int_y^0 \lambda(u)du} \, dy -\int_0^\infty e^{-\int_0^y \lambda(u)du} \, dy\right)\) is strictly monotone decreasing in \(b\). Accordingly, the derivative \(\frac{\partial E|\tau_b|}{\partial b}\) can change sign at most once, and only from negative to positive. Therefore, if there is an extreme point, it must be a global minimum.
\end{proof}
In the following we rewrite the LHS of \eqref{criteria_optimal} in terms of the integrated jump intensity 
\[\Lambda(y) := \int_y^0 \lambda(u) \, du, \,\, y \in [S,0].\]
Hence, if $b^*$ satisfies the equation 
\begin{align}
    \int_{b^*}^0 e^{\Lambda(y)} \, dy =\int_0^\infty e^{-\int_0^y \lambda(u)du}  \, dy,
\end{align}
then $b^*$ is an optimal threshold. 

We finally remark that since $\lambda$ is assumed to be greater than or equal to a positive constant, there exists at most one real $b^* \in [S,0]$ satisfying equation \eqref{criteria_optimal}. If there is no such \(b^*\in [S,0]\), then it follows directly from the proof of Theorem \ref{criteria} that \(S\) is an optimal threshold.

\subsubsection*{Intuition behind the optimality criterion}
Note that for all \(t>0\)
\[P(\tau_0>t)=P(N_t^\lambda-N_0^\lambda =0)=e^{-\int_0^t \lambda(u)du},\] and hence the RHS of \eqref{criteria_optimal} coincides with the expectation \(E[\tau_{b^*} | \tau_{b^*} > 0] = E(\tau_0)=\int_0^\infty P(\tau_0>t) \, dt.\) In other words, the right-hand side of \eqref{criteria_optimal} corresponds exactly to the expected value of the first jump time after the parking space 0. The right-hand side, therefore, measures the expected costs after the driver has passed parking space 0. 

The left-hand side of Equation \eqref{criteria_optimal} is equal to 
$\frac{|b^*| - E[|\tau_{b^*}| 1_{\{\tau_{b^*} < 0\}} ]}{P(\tau_{b^*} \ge 0)}$. Rearranging terms yields that \eqref{criteria_optimal} is equivalent to
\begin{align}\label{indifference}
    |b^*| = E|\tau_{b^*}|.  
\end{align}
Equation \eqref{indifference} characterizes $b^*$ as the position at which the agent is indifferent between taking the lot, provided it is free, and continuing until the next free lot. 

\subsubsection*{Optimality gap}

We denote by 
\begin{align}\label{defi og}
\Delta(b):= E|\tau_b| - E|\tau_{b^*}|    
\end{align}
the optimality gap of choosing the stopping rule with threshold $b \in [S, 0]$ instead of the optimal threshold $b^*$. The next lemma collects some properties of the optimality gap. 
\begin{lemma}\label{suff condi Delta smooth}
$\Delta$ is differentiable. Moreover, if $\lambda$ is continuously differentiable, then $\Delta$ is twice continuously differentiable, and the second derivative is given by 
\begin{align}\label{sec deri Delta}
& \Delta''(b)  \\
= & \lambda'(b)\left[b-\int_b^0 y \lambda(y) e^{-\int_b^y \lambda(u)du} \, dy + \int_0^\infty y \lambda(y) e^{-\int_b^y \lambda(u)du} \, dy \right] \nonumber\\
& + \lambda(b)\left[1 + b \lambda(b) - \int_b^0 y \lambda(y) \lambda(b) e^{-\int_b^y \lambda(u)du} \, dy + \int_0^\infty y \lambda(y) \lambda(b) e^{-\int_b^y \lambda(u)du} \, dy \right]. \nonumber 
\end{align}
\end{lemma}
\begin{proof}
The claims are straightforward to show. 
\end{proof}
We close this section by considering the special case where $\lambda$ is constant. 
\begin{corollary}
\label{constant_lambda_cost}
Suppose that $\lambda$ is constant and greater than $\ln(2)/|S|$. Then $b^* = - \ln(2)/\lambda$ is an optimal threshold and 
\begin{align}
    \Delta_\lambda(b) = \frac{1}{\lambda}(2 e^{\lambda b} -1) -b-\frac{\ln(2)}{\lambda}.
\end{align}
\end{corollary}

\section{Learning optimal rules}\label{sec 2}

Based on the Equation \eqref{criteria_optimal}, the optimal stopping time \(\tau_{b^*}\) can be determined with the knowledge of the intensity function \(\lambda\). In the following we assume that the agent does not know the intensity function \(\lambda\). We do assume, however, that the agent has to solve the stopping problem in many consecutive rounds and that in each round she observes the jump process up to the chosen stopping time. With the observations the agent can estimate the true intensity function $\lambda$. Hence, in round $n$ the agent can approximate the optimal stopping rule by using the observations made in previous rounds $0,1, \ldots, n-1$.

Let \((N^{0}_t)_{t\in [S, \infty)} , (N^{1}_t)_{t\in [S, \infty)}, \dots\) be a sequence of stochastic processes on a measurable space $(\Omega, \cF)$. Assume that for every continuous jump intensity function \(\lambda:[S, \infty) \to [c,\infty)\), where \(c\) is a positive constant, there exists a probability measure $P_\lambda$ such that \((N^{0}_t)_{t\in [S, \infty)} , (N^{1}_t)_{t\in [S, \infty)}, \dots\) is an independent sequence of Poisson processes with jump intensity \(\lambda\). 

Similarly to the previous section, we denote by $\tau^{i}_{b}$ the first jump time of $N^i$ after time $b \in [S, 0]$.  Moreover, the optimality gap under $P_\lambda$ is denoted by $\Delta_\lambda(b)$.

\begin{defi}[Policy]\label{defi policy}
A policy \((\pi_n)_{n\ge0}\) is a sequence of random variables with values in \([S,0]\) such that \(\pi_n\) is measurable with respect to \(\cF_{n},\) where $\cF_0 = \{\emptyset, \Omega\}$ and \(\cF_{n} := \bigvee_{k=0}^{n-1} \sigma (N^{k}_t: t \in [S, \tau^k_{\pi_{k}}])\) for every \(n\ge 1.\) (Recall that if $(\mathcal A_i)_{i \in I}$ is a family of $\sigma$-algebras indexed by a set $I$, then $\bigvee_{i \in I} \cA_i$ denotes the smallest $\sigma$-algebra containing all $\mathcal A_i$.) We denote the set of policies by \(\Pi.\)
\end{defi}
We interpret \(\pi_n\) as the threshold the agent chooses in round \(n\) for the stopping rule. 
\\ \\
For a given positive and measurable intensity function \(\lambda\), we define the regret of an arbitrary policy \((\pi_n)_{n\ge0}\) in round \(T \in \mathbb{N}\) as
\[
\mathcal{R}^{\pi}_\lambda(T):=\sum_{n=0}^T E_\lambda[\Delta_\lambda (\pi_n)], 
\]
where $\Delta_\lambda$ is the optimality gap function defined in \eqref{defi og}. 

\begin{defi}[Epoch]
    For every \(l \in \mathbb{N}_0\), we define \(\cT_l:=\{2^l-1,\dots,2^{l+1}-2\}.\) An epoch \(l \in \mathbb{N}_0\) consists exactly of those rounds that are contained in the set \(\cT_l\). Thus, each epoch \(l\) consists of \(|\cT_l|=2^l\) rounds.
\end{defi}

For example, \(\cT_0=\{0\}, \cT_1=\{1,2\}\) and \(\cT_2=\{3,4,5,6\}\).

Next we describe a specific algorithm that leads to a policy with a regret growing logarithmically as the number of rounds converges to infinity. We later prove that there is no policy that can achieve a regret growing slower than the logarithm. In this sense the algorithm is asymptotically optimal. 

\begin{algorithm} [{\bf Indifference level updating (ILU)}]
\mbox{} 
\begin{enumerate}
\item Choose the hyperparameter \(L\in (1,\infty).\) Define \(L_{low}:=\frac{\ln(2)}{|S|}+ \frac1L.\) We assume \(L>L_{low}\).
\item Initialize $\hat b_0:=0$. 
\item Choose the stopping time with threshold $0$ in round 0.
\item For every Epoch \(l\ge 1\) do the following: 
\begin{enumerate}
    \item Let \(m:=2^{l-1}\) and \( b := \hat b _{l-1}\). 
     \item Compute 
    \begin{align}
    \label{empSurvival}
        \hat G (y):= \frac1m \sum_{i \in \cT_{l-1}} \mathbf{1}_{\{\tau_b^i>y\}}, \quad y\in[b,0]
    \end{align}
    \item Compute 
    \begin{align}
\label{defi_int_jump}
\hat I (y):=\frac{\sum_{i \in \cT_{l-1}} (N^i_b - N^i_y)}{m}, \quad y \in [S, b], 
\end{align}
\item Set
\begin{align}
    \label{lambda_uncuttet}
    \hat \Lambda_l^{uc}(y):=\begin{cases}
        0 & \text{for } y= 0,\\
    \ln \hat G(y) - \ln \hat G (0) & \text{for } b < y <0,\\
    - \ln \hat G (0) + \hat I(y) & \text{for } S \le y \le b
    \end{cases}
\end{align}
with the convention \(\ln(0)=-\infty\) and \(-\infty+\infty:=+\infty\), and 
\begin{align}
    \label{lambda_cut}
    \hat \Lambda_l := \min\{\hat \Lambda^{uc}_l, L |S|\}.
\end{align}
\item Let \(K:= m\hat G(0)\). Compute 
\begin{align}
\label{defi_jump_est}
\hat{\phi}_l := \begin{cases}
\min\{\max\{\frac1L,\frac{1}{K} \sum_{i \in \cT_{l-1}} \tau^{i}_0 \, \mathbf{1}_{\{\tau_b^i >0\}}\},\frac{1}{L_{low}}\} \, ,& \quad \text{for } K\ge 1, \\
\frac1L,& \quad K=0.
\end{cases}
\end{align}
\item Determine  $\hat{b}_l$ such that 
\begin{align}
\label{determinebhat}
\int_{\hat{b}_l}^0 e^{\hat{\Lambda}_l(y)} \, dy = \hat{\phi}_l.
\end{align}
If there is no solution \(\hat{b}_l \in [S,0]\), then set \(\hat{b}_l=S\). 
\item Choose, for the whole Epoch \(l\), the threshold \(\hat b _l\). 
\end{enumerate}
\end{enumerate} 
\end{algorithm}

Note that for \(y \in [b,0]\), we have \(G(y):=P_\lambda(\tau_b>y)=e^{-\int_b^y\lambda(u)du}=e^{-(\Lambda(b)-\Lambda(y))}.\) In particular, \(G(0)=e^{-\Lambda(b)}\ge e^{-L|S|}\) and \(G(b)=1\). Therefore
\begin{align}
\label{jump_int_via_surv}
   \Lambda(y)=\begin{cases}
       \ln G(y)-\ln G(0),& \quad y>b,\\
       -\ln G(0) + (\Lambda(y)-\Lambda(b)), & \quad y \le b.
   \end{cases}
\end{align}
The estimator \(\hat G\) from the ILU algorithm is an estimator for the survival probability \(G\). Moreover, \(\hat I\) is the estimator for the increments of the integrated jump intensity before reaching the chosen threshold, i.e. the estimator for \(\Lambda(y)-\Lambda(b)\) for \(y \in [S,b]\). Lastly, \(\hat \Lambda\) is exactly the empirical counterpart of the representation of the integrated jump intensity in \eqref{jump_int_via_surv}.

Note that \(\tau_b^i\) is always observable because the threshold \(b\) used for generating the observations is fixed during an epoch. In addition, the Poisson process is always observable until the (fixed) threshold is reached. It follows that in every epoch \(l\), the same number of observations is gathered for the estimators of the survival probability and the increments of the integrated jump intensity in the next epoch \(l+1\), namely \(|\cT_l |\) many observations.

Lastly, note that \(K\) from the ILU algorithm is the number of rounds where the agent observes the first free parking lot after 0 and therefore an observation for estimating \(\tau_0^i\) is accessible. 

The sequence of thresholds chosen by the ILU algorithm is a policy in the sense of Definition \ref{defi policy}.  
\begin{defi}
    Let \(\cA_l\) be the \(\sigma\)-algebra generated by the observations of the ILU algorithm during the epochs \(0, ..., l-1\). In particular, if \((\hat b_i)_{i\in\{0,\dots,l-1\}}\) is the sequence of thresholds generated by the ILU algorithm during the first \(l\) epochs, then \(\cA_l:=\bigvee_{k=0}^{2^l-2} \sigma (N^{k}_t: t \in [S, \tau^k_{\hat b_{\lfloor \log_2(k+1)\rfloor }}])\). 
\end{defi}
\section{Main results}
\label{mainresults_chapter}

We start by defining an environment class of smooth intensity functions \(\lambda\).

\begin{defi}[Environment class \(\mathcal{M}(L)\), $L \in (1,\infty)$]
\label{smooth_env_class}
\leavevmode\\
We say \(\lambda \in \mathcal{M}(L)\) if and only if
\begin{enumerate}[label=(\arabic*)]
\item \(\lambda\) is continuously differentiable;
\item \label{Smooth_Property_LB} \(\lambda\) is bounded from below by \(\ln(2)/|S| + 1/L=:L_{low};\)
\item \(\lambda\) and $\lambda'$ are bounded by $L$ on \( [S,\infty) \), i.e. \(\lambda(u) \le L\) and $|\lambda'(u)| \le L$, for all \(u\in [S,\infty)\). 
\end{enumerate}
\end{defi}
Note that we assume in the following that \(\lambda \in \cM(L)\) and therefore, even though the true \(\lambda\) is unknown, we know that \(E_\lambda[\tau_0]\in [\frac1L, \frac{1}{L_{low}}] \) which is exactly the reason for clipping of the estimator \(\hat\phi\) in the ILU algorithm.\\ \\
Notice that property \ref{Smooth_Property_LB} guarantees the following.  

\begin{lemma}\label{11Jan2026}
For any $\lambda \in \cM(L)$ there exists a unique $b^* \in (S,0)$ satisfying \eqref{criteria_optimal}; moreover this $b^*$ satisfies $b^* - S \ge |S|-L_{low}^{-1}\ln(2)=|S|-\left(\frac{\ln(2)}{|S|}+\frac1L\right)^{-1} \ln(2) >0. $
\end{lemma}
\begin{proof}
The smallest possible value for $b^*$, denoted by \(b^*_{min}\), is achieved by $\lambda \in \cM(L)$ that is constant equal to the minimal value. Hence \(b^*_{\min}\) satisfies
\begin{align}
\label{b_min_log}
& \int_{b^*_{min}}^0 e^{-(\frac{\ln(2)}{|S|}+\frac1L)u} du = \left(\frac{\ln(2)}{|S|}+\frac1L\right)^{-1}. 
\end{align}
A straightforward computation shows that
\begin{align*}
b^*_{min} = -\left(\frac{\ln(2)}{|S|}+\frac1L\right)^{-1} \ln(2) > -\frac{|S|}{\ln(2)}\ln(2)=S.   
\end{align*}

\end{proof}

In the following we denote by $\cR^{ILU(L)}_\lambda(T)$ the regret in round $T \in \N$ entailed by the ILU algorithm with hyperparameter \(L\) from the corresponding environment class under the probability measure $P_\lambda$. 

\begin{theo}[Upper bound of ILU on class $\cM(L)$]
\label{main_res_1}
There exists a constant \(C\in \mathbb{R}\), depending only on $S$ and $L$, such that for all \(T\in \N,\) we have
\begin{align*}
\sup_{\lambda\in \cM(L)} \mathcal{R}^{ILU(L)}_\lambda(T) \le C \ln(T+2).
\end{align*}
\end{theo}

The proof is provided in Section \ref{sec proof main res 1}.

\begin{theo}[Lower bound on class $\cM(L)$]
\label{main_res_3}
Suppose that $L$ is larger than $\ln(2)/|S|+1/L$, and hence $\cM(L)$ is non-empty. Then there exists $c \in (0,\infty)$ such that for all $T \in \N$
\[
\inf_{\text{policy } \pi} \sup_{\lambda \in \mathcal{M}(L)} 
\mathcal{R}^\pi_\lambda(T) \ge c \ln(T).
\]
\end{theo}
The proof is provided in Section \ref{lowerbound}.\\
Theorem \ref{main_res_1} and Theorem \ref{main_res_3} imply that the ILU algorithm has an asymptotically optimal growth rate.

\section{Preliminary Considerations}
\label{sec_abstract}
Let \(\lambda \in \mathcal{M}(L)\). To simplify notation we omit the superscript zero and write \(\tau_b = \tau^0_b\). Note that $\tau_b$ has the same distribution as $\tau_b^i$ for any $i \ge 1$, under any $P_\lambda$.
To simplify notation we omit $\lambda$ in $P_\lambda$, $E_\lambda$, $\Var_\lambda$ etc. Recall that \(b^*\) is the optimal threshold with regard to the true but unknown intensity \(\lambda\).
\\

For this section, we fix some arbitrary constants \(c_1,c_2, c_3\) and \(n \in \mathbb{N}.\) Moreover, let \(\tilde{\Lambda}:[S,0]\to[0,\infty)\) with \(\tilde\Lambda(0)=0\) be a measurable and integrable function estimator that is \(P\)-a.s. non-increasing and satisfies 
\begin{align}
    \label{H1}
    E[||\tilde\Lambda - \Lambda||^2_\infty]\le c_1/n
\end{align}
and 
\begin{align}
    \label{H2}
    E[(\tilde\Lambda(y)-\Lambda(y))^2]\le c_2 /n
\end{align}
for all \(y \in [S,0]\).

Moreover, let \(\tilde \phi >0\) be an arbitrary scalar that fulfills
\begin{align}
    \label{H3}
    E[(\tilde \phi - E[\tau_0])^2]\le c_3/n.
\end{align}
Define \(\tilde b \in [S,0]\) by 
\begin{align}
\label{new_b_tilde}
  \int_{\tilde b} ^0 e^{\tilde\Lambda(y)}\, dy=\tilde \phi,  
\end{align}
and \(\tilde b=S\) if \eqref{new_b_tilde} has no solution in \([S,0].\)

\begin{proposition}
\label{abstract_mse_b}
With \(M:=\min\{m\in \mathbb{N}_0: S+m\frac{1}{2L} \ge 0\}\) we have
\begin{align}
\label{MSEb_abstract}
E\left(\left(\tilde b-b^*\right)^2\right) &\le \left( S^2 \frac{c_1+c_3}{\epsilon^2} \right)\frac1n + \\  
&\quad +\left( S^2 4 c_2 (M+1) + 2 c_3 + 8 |S| c_2 \int_S^0 e^{2\Lambda(y)}\, dy  \right) \frac1n, \notag 
\end{align}
where \(0<\epsilon \le \min\{E(\tau_0),1,\frac{b^*-S}{3+2 E(\tau_0)}\}.\)
\end{proposition}
\begin{proof}
Fix $n \in \N$ and suppose first that we are on the event \(B:=\{ |S| > |\tilde{b}|\}\). 

Note that \(\tilde{b}\) is determined in such a way that
\[\int_{\tilde{b}}^0 e^{\tilde{\Lambda}(y)} \, dy =\tilde{\phi}.\]
Furthermore, the optimal \(b^*\) and the unknown, true intensity function \(\lambda\) satisfy \[\int_{b^*}^0 e^{\Lambda(y)} \, dy = E(\tau_0).\]
Therefore
\begin{align*}
\int_{b^*}^0 e^{\Lambda(y)} \, dy - \int_{\tilde{b}}^0 e^{\tilde{\Lambda}(y)} \, dy= E(\tau_0) - \tilde{\phi}.
\end{align*}
This is equivalent to
\begin{align}
\label{eq1schätzer}
\int_{b^*}^0 e^{\Lambda(y)} \, dy - \int_{\tilde{b}}^0 e^{\Lambda(y)} \, dy + \int_{\tilde{b}}^0 \left(e^{\Lambda(y)}-e^{\tilde{\Lambda}(y)}\right) \, dy= E(\tau_0) - \tilde{\phi}.
\end{align}
\underline{\(1^{st}\) case:} Let \(\tilde{b} \ge b^*\). It follows from \eqref{eq1schätzer}:
\begin{align}
\label{eq2schätzer}
\int_{b^*}^{\tilde{b}} e^{\Lambda(y)} \, dy &= E(\tau_0) - \tilde{\phi} - \int_{\tilde{b}}^0 \left(e^{\Lambda(y)}-e^{\tilde{\Lambda}(y)}\right) \, dy \notag\\
&= E(\tau_0) - \tilde{\phi} + \int_{\tilde{b}}^0 e^{\Lambda(y)}\left(e^{\tilde{\Lambda}(y)-\Lambda(y)}-1\right) \, dy.
\end{align}
In the following, we estimate the left-hand side of equation \eqref{eq2schätzer} from below and the right-hand side from above. We define for this 
\begin{align*}
A&:=\int_{b^*}^{\tilde{b}} e^{\Lambda(y)} \, dy,\\
B&:=E(\tau_0) - \tilde{\phi} + \int_{\tilde{b}}^0 e^{\Lambda(y)}\left(e^{\tilde{\Lambda}(y)-\Lambda(y)}-1\right) \, dy.
\end{align*}
Note that
\[A \ge (\tilde{b}-b^*) \min_{y\in [S,0]} e^{\Lambda(y)} \ge \tilde{b}-b^*;\]
\[B \le  \left |E(\tau_0) - \tilde{\phi}\right| + \int_{S}^0 e^{\Lambda(y)}\left|e^{\tilde{\Lambda}(y)-\Lambda(y)}-1\right| \, dy.\]
Observe that for all \(x\in [-1,1]\) we have
\begin{align}
\label{K_MSE}
\left|e^x-1\right| \le 2 |x|.
\end{align}

Define \(C := \left\{\left|\tilde{\Lambda}(y)-\Lambda(y)\right| \le 1, \forall y \in [S,0]\right \}.\)\\
Then on \(C\) we have \[\left|e^{\tilde{\Lambda}(y)-\Lambda(y)}-1\right|\le 2\left|\tilde{\Lambda}(y)-\Lambda(y)\right|.\]
In particular, using the estimates of A and B, we get
\begin{align*}
\tilde{b}-b^* \le  \left |E(\tau_0) - \tilde{\phi}\right| + 2 \int_S^0 e^{\Lambda(y)} \left|\tilde{\Lambda}(y)-\Lambda(y)\right|\, dy.
\end{align*}
\underline{\(2^{nd}\) case:} Let \(\tilde{b} < b^*\). Then \eqref{eq1schätzer} implies
\begin{align*}
-\int_{\tilde{b}}^{b^*} e^{\Lambda(y)} \, dy = E(\tau_0) - \tilde{\phi} + \int_{\tilde{b}}^0 e^{\Lambda(y)}\left(e^{\tilde{\Lambda}(y)-\Lambda(y)}-1\right) \, dy,
\end{align*}
which is equivalent to 
\begin{align}
\label{eq2schätzercase2}
\int_{\tilde{b}}^{b^*} e^{\Lambda(y)} \, dy =(\tilde{\phi} - E(\tau_0)) + \int_{\tilde{b}}^0 e^{\Lambda(y)}\left(-\left(e^{\tilde{\Lambda}(y)-\Lambda(y)}-1\right)\right) \, dy.
\end{align}
In the following, we estimate the left-hand side of equation \eqref{eq2schätzercase2} from below and the right-hand side from above. We define for this 
\begin{align*}
A&:=\int_{\tilde{b}}^{b^*} e^{\Lambda(y)} \, dy,\\
B&:=(\tilde{\phi} - E(\tau_0)) + \int_{\tilde{b}}^0 e^{\Lambda(y)}\left(-\left(e^{\tilde{\Lambda}(y)-\Lambda(y)}-1\right)\right) \, dy.
\end{align*}
Then
\[A \ge (b^*-\tilde{b}) \min_{y\in [S,0]} e^{\Lambda(y)}\ge b^*-\tilde{b};\]
\[B \le  \left |E(\tau_0) - \tilde{\phi}\right| + \int_{S}^0 e^{\Lambda(y)}\left|e^{\tilde{\Lambda}(y)-\Lambda(y)}-1\right| \, dy.\]
Let \(C\) be as in the first case. Then on \(C\) we have \[\left|e^{\tilde{\Lambda}(y)-\Lambda(y)}-1\right|\le 2\left|\tilde{\Lambda}(y)-\Lambda(y)\right|.\]
Therefore applies
\begin{align*}
b^* - \tilde{b} \le \left |E(\tau_0) - \tilde{\phi}\right| + 2 \int_S^0 e^{\Lambda(y)} \left|\tilde{\Lambda}(y)-\Lambda(y)\right|\, dy.
\end{align*}
To sum up, in any case, if $B$ and \(C\) occur, we have
\begin{align}
\label{firstequationMSE}
\left| \tilde{b}-b^*\right| \le \left |E(\tau_0) - \tilde{\phi}\right| + 2 \int_S^0 e^{\Lambda(y)} \left|\tilde{\Lambda}(y)-\Lambda(y)\right|\, dy.
\end{align}
Thus, using the estimate \((e+f)^2 \le 2(e^2+f^2), \, e, f \in \mathbb{R}\)
\begin{align}
\label{secondequationMSE}
E\left(\left|\tilde{b}-b^*\right|^2 \mathbf{1}_{B} \mathbf{1}_{C} \right) \le \,& 2 \Big \{ E\left[\left(E(\tau_0) - \tilde{\phi}\right)^2 \mathbf{1}_{B} \mathbf{1}_{C}\right] +\\
&\quad + 4 E\left[\mathbf{1}_{B} \mathbf{1}_{C} \left(\int_S^0e^{\Lambda(y)} \left|\tilde{\Lambda}(y)-\Lambda(y)\right|\, dy\right)^2\right]\Big\}. \notag \\
&\le 2 \Big \{ E\left[\left(E(\tau_0) - \tilde{\phi}\right)^2 \right] + \notag\\
& \quad+ 4 E\left[\left(\int_S^0e^{\Lambda(y)} \left|\tilde{\Lambda}(y)-\Lambda(y)\right|\, dy\right)^2\right]\Big\}. \notag
\end{align}
Using Jensen's inequality for integrals we further obtain
\begin{align*}
& E\left(\left|\tilde{b}-b^*\right|^2 \mathbf{1}_{B} \mathbf{1}_{C} \right) \\
&\le 2\Big \{ E\left[\left(E(\tau_0) - \tilde{\phi}\right)^2\right] + 4 E\left[|S| \left(\int_S^0 \left( e^{\Lambda(y)} \left|\tilde{\Lambda}(y)-\Lambda(y)\right| \right)^2 \, dy\right)\right]\Big\}\\
&= 2 \Big \{ \MSE(\tilde{\phi}) + 4 |S| \int_S^0 e^{2\Lambda(y)} \MSE(\tilde{\Lambda}(y))\, dy \Big\}.\\
\end{align*}
With the properties \eqref{H2} and \eqref{H3} we further obtain that
\begin{align}
\label{D_MSE}
E\left(\left|\tilde{b}-b^*\right|^2 \mathbf{1}_{B} \mathbf{1}_{C} \right) \le  \frac1n \left(2 c_3 + 8 |S| c_2 \int_S^0 e^{2\Lambda(y)} \, dy \right).
\end{align}
In summary, we have
\begin{align}
\label{EventBnfromMSE}
& E((\tilde{b}-b^*)^2) \nonumber\\
&=E((\tilde{b}-b^*)^2 \mathbf{1}_{B^c}) + E((\tilde{b}-b^*)^2 \mathbf{1}_{B} \mathbf{1}_{C^c})+E((\tilde{b}-b^*)^2 \mathbf{1}_{B} \mathbf{1}_{C}) \notag \\
&\le S^2 P(B^c) + S^2 P(C^c)+ \frac1n \left(2 c_3 + 8 |S| c_2 \int_S^0 e^{2\Lambda(y)}\, dy \right). 
\end{align}

Due to Lemma \ref{new_all_in_ball} below we have
\begin{align}
\label{estimateB}
    P(B^c) &\le \frac{c_1 + c_3}{\epsilon^2} \frac1n
\end{align}
where \(0<\epsilon \le \min\{E(\tau_0),1,\frac{b^*-S}{3+2 E(\tau_0)}\}.\)

Furthermore, with the Estimate \eqref{PAnC_estimate} from Lemma \ref{ordnungAnC} below, we have
\begin{align}
\label{MSE_E}
P(C^c) \le \left( 4 c_2  (M+1)\right)\frac1n. 
\end{align}
By combining the estimates \eqref{EventBnfromMSE}, \eqref{estimateB} and \eqref{MSE_E} we get \eqref{MSEb_abstract}.
\end{proof} \noindent

\begin{lemma}
\label{new_all_in_ball}
Let \(B:=\{ |S| > |\tilde{b}|\}\). Then we have \[P(B^c)\le \frac{c_1+c_3}{\epsilon^2n},\]
where \(0<\epsilon \le \min\{E(\tau_0),1,\frac{b^*-S}{3+2 E(\tau_0)}\}.\)
\end{lemma}
\begin{proof}
We want to determine \(\epsilon:=\epsilon(\lambda) >0\) such that \(||\tilde{\Lambda}-\Lambda||_\infty < \epsilon\) and \(|E(\tau_0)-\tilde{\phi}| < \epsilon\) implies \(|\tilde{b}| < |S|\). Let \(\delta := b^* - S\). \\

We show via contradiction that \(\epsilon < \min\{1,\frac{\delta}{3+2 E(\tau_0)}\}\) implies \(|\tilde{b}| < |S|\). \\

Assume \(|\tilde{b}| = |S|\). We have 
\begin{align*}
\int_{\tilde{b}}^0 e^{\tilde{\Lambda}(u)} du &\ge \int_{\tilde{b}}^0 e^{\Lambda(u)-\epsilon} du = \int_S^0  e^{\Lambda(u)-\epsilon} du = e^{-\epsilon}\int_{b^*-\delta}^0  e^{\Lambda(u)} du \\
&=e^{-\epsilon}\left(\int_{b^*}^0  e^{\Lambda(u)} du + \int_{b^*-\delta}^{b^*}  e^{\Lambda(u)} du\right) = e^{-\epsilon}\left(E(\tau_0) + \int_{b^*-\delta}^{b^*}  e^{\Lambda(u)} du\right) \\
&\ge e^{-\epsilon}\left(E(\tau_0) + \delta \min_{u \in [S,0]} e^{\Lambda(u)} \right) = e^{-\epsilon}\left(E(\tau_0) + \delta \right).
\end{align*}
Moreover 
\begin{align*}
E(\tau_0) + \epsilon &\ge \tilde{\phi} \ge \int_{\tilde{b}}^0 e^{\tilde{\Lambda}(u)} du \ge e^{-\epsilon}\left(E(\tau_0) + \delta \right).
\end{align*}
Therefore we have a contradiction if \( e^{\epsilon}(E(\tau_0) + \epsilon) < \left(E(\tau_0) + \delta \right)\). Note that for \(0 \le \epsilon \le 1\) it holds that \(e^\epsilon \le 1+\epsilon+\epsilon^2\). It follows for \(\epsilon \in [0,1]\)
\begin{align*}
e^{\epsilon}(E(\tau_0) + \epsilon)&\le (1+\epsilon+\epsilon^2)(E(\tau_0) + \epsilon) \le (1+2\epsilon)(E(\tau_0) + \epsilon) \\
&\le E(\tau_0) +\epsilon+ 2 \epsilon E(\tau_0) + 2 \epsilon ^2 \\
&\le E(\tau_0) +(3+ 2 E(\tau_0)) \epsilon.
\end{align*} 
With \(\epsilon < \min\{1,\frac{\delta}{3+2 E(\tau_0)}\}\) we have \(e^{\epsilon}(E(\tau_0) + \epsilon) \le E(\tau_0) +(3+ 2 E(\tau_0)) \epsilon < E(\tau_0) + \delta\) and therefore a contradiction.

Note that \(E(\tau_0) - \epsilon \ge 0\) and therefore we have in the end \(\epsilon < \min\{E(\tau_0),1,\frac{\delta}{3+2 E(\tau_0)}\}.\)
Finally, 
\begin{align*}
P(B^c)\le P(||\tilde{\Lambda}-\Lambda||_\infty \ge \epsilon) + P(|E(\tau_0)-\tilde{\phi}| \ge \epsilon) \le \frac{E(||\tilde{\Lambda}-\Lambda||^2_\infty) + E_\lambda[(\tilde \phi - E_\lambda[\tau_0])^2]}{\epsilon^2}.
\end{align*}
Since \(P(B^c)\) does not depend on \(\epsilon\), the inequality holds as well for \(\epsilon=\min\{E(\tau_0),1,\frac{\delta}{3+2 E(\tau_0)}\}\). With properties \eqref{H1} and \eqref{H3} the statement directly follows.
\end{proof}
In the following let \(\tilde{C}(y,r) := \left\{\left|\tilde{\Lambda}(y)- \Lambda(y)\right| \le r\right\} \, , r >0\). In particular,\[C=\bigcap_{y\in [S,0]} \tilde{C}(y,1),\]
where the event $C$ is defined as in the proof of Proposition \ref{MSEb_abstract}.    
\begin{lemma}
\label{hilfslemma1}
Let \( y \in [S,0]\) and define \(k:= \frac{1}{2L}\). Then for all \(t \in [y,(y+k) \wedge 0]\) and \(\omega \in \tilde{C} (y,\frac{1}{2}) \cap \tilde{C} (y+k,\frac{1}{2}):\)\[\left|\tilde{\Lambda}(t)-\Lambda(t)\right| \le 1.\]
\end{lemma}
\begin{proof}
W.l.o.g. let \(y+k < 0\). If \(y+k \ge 0\), then replace \(y+k\) by 0 in the following considerations and the estimates are still valid.\\
Let \(\omega \in \tilde{C} (y,\frac{1}{2}) \cap \tilde{C} (y+k,\frac{1}{2})\). It follows that
\(\left|\tilde{\Lambda}(y,\omega)-\Lambda(y)\right| \le \frac{1}{2}\) and \(\left|\tilde{\Lambda}(y+k,\omega)-\Lambda(y+k)\right| \le \frac{1}{2}\). Moreover, it follows directly from the definition that \(\tilde{\Lambda}(y_1,\omega)\ge \tilde{\Lambda}(y_2,\omega)\) if \(y_1\le y_2\) and \(y_1,y_2 \in [S,0]\). Therefore, the following applies for any \(t \in [y,y+k]\): 
\begin{align*}
\tilde{\Lambda}(t,\omega)-\Lambda(t)&\le \tilde{\Lambda}(y,\omega)+\Lambda(y)-\Lambda(y)-\Lambda(t)\\
&\le \frac{1}{2} + L (t-y) \\
&\le \frac{1}{2} + L k = 1. 
\end{align*}
Furthermore 
\begin{align*}
\Lambda(t)-\tilde{\Lambda}(t,\omega)&\le \Lambda(t)-\Lambda(y+k)+\Lambda(y+k)- \tilde{\Lambda}(y+k,\omega)\\
&\le L k + \frac{1}{2}= 1. 
\end{align*}
\end{proof}
\begin{lemma}
\label{ordnungAnC}
Let \(M:=\min\{m\in \mathbb{N}_0: S+m\frac{1}{2L} \ge 0\}\). Then 
\begin{align}
\label{PAnC_estimate}
P(C^c) \le 4 c_2  (M+1) \frac1n,
\end{align}
where \(C = \left\{\left|\tilde{\Lambda}(y)-\Lambda(y)\right| \le 1, \forall y \in [S,0]\right \}.\)
\end{lemma}
\begin{proof}
We have \(M=\min\{m\in \mathbb{N}_0: S+m \frac{1}{2L} \ge 0\}\). Define \(y_m := S+m\frac{1}{2L}\) with \( m\in \{0,1,\dots,M-1\}\). Furthermore let \(y_M:=0\). According to Lemma \ref{hilfslemma1} it holds for all \(t \in [S,0]\) and \(\omega \in \tilde{C}(y_0, \frac{1}{2}) \cap \tilde{C}(y_1, \frac{1}{2}) \cap \dots \cap \tilde{C}(y_M, \frac{1}{2})\): \[\left|\tilde{\Lambda}(t)-\Lambda(t)\right| \le 1.\]
In particular, the following applies:
\begin{align*}
\tilde{C}\Big(y_0, \frac{1}{2}\Big) \cap \dots \cap \tilde{C}\Big(y_M, \frac{1}{2}\Big) & \subseteq   \left\{\left|\tilde{\Lambda}(y)-\Lambda(y)\right| \le 1, \forall y \in [S,0]\right\}\\
&= C.
\end{align*}
Therefore, using \(C^c \subseteq \left(\tilde{C}(y_0, \frac{1}{2}) \cap \tilde{C}(y_1, \frac{1}{2}) \cap \dots \cap \tilde{C}(y_M, \frac{1}{2})\right)^c\) and the sub-additivity of probability measures, we get
\begin{align*} 
P(C^c) &\le  P\left(\left(\tilde{C}\Big(y_0, \frac{1}{2}\Big) \cap \tilde{C}\Big(y_1, \frac{1}{2}\Big) \cap \dots \cap \tilde{C}\Big(y_M, \frac{1}{2}\Big)\right)^c\right) \notag \\
&\le \sum_{i=0}^M P\left(\tilde{C}\Big(y_i,\frac{1}{2}\Big)^c\right) \notag \\
&= \sum_{i=0}^M P\left(\left\{\left|\tilde{\Lambda}(y_i)- \Lambda(y_i)\right| > \frac{1}{2}\right\}\right) \notag \\
&\le \sum_{i=0}^M \frac{1}{(1/2)^2} E_\lambda[(\tilde\Lambda(y_i)-\Lambda(y_i))^2] \notag \\
&= \sum_{i=0}^M 4 c_2 \frac1n \notag \\
&\le 4 c_2 (M+1) \frac1n.
\end{align*}
\end{proof}

\begin{corollary}
\label{cor_uniform_constant}
It holds 
\begin{align*}
E\left(\left(\tilde b-b^*\right)^2\right) &\le \frac1n C_{uni},
\end{align*}
where \(C_{uni}= S^2 \frac{c_1+c_3}{\left(\min\{\frac1L,\frac{|S|-L_{low}^{-1}\ln(2)}{3+2 L_{low}^{-1}}\}\right)^2}+ S^2 4 c_2 (2 |S| L+2) + 2 c_3 + 8 S^2 c_2e^{2L|S|}  .\)

\end{corollary}
\begin{proof}
From Proposition \ref{abstract_mse_b} we know that
\begin{align*}
E\left(\left(\tilde b-b^*\right)^2\right) &\le \left( S^2 \frac{c_1+c_3}{\epsilon^2} \right)\frac1n + \\  
&\quad +\left( S^2 4 c_2 (M+1) + 2 c_3 + 8 |S| c_2 \int_S^0 e^{2\Lambda(y)}\, dy  \right) \frac1n, \notag 
\end{align*}
where \(0<\epsilon \le  \min\{E(\tau_0),1,\frac{b^*-S}{3+2 E(\tau_0)}\}\) and \(M:=\min\{m\in \mathbb{N}_0: S+m\frac{1}{2L} \ge 0\}\).

With the estimates
\begin{enumerate}
    \item \(\min\{m\in \mathbb{N}_0: S+m \frac{1}{2L} \ge 0\} \le\lceil 2 |S|L \rceil \le 2 |S| L+1;\)
     \item \(\Lambda(y)\le |S| L,\)
\end{enumerate}
where \(y\in[S,0]\) and the estimate directly following from Lemma \ref{11Jan2026}
\begin{align*}
    \frac{b^*-S}{3+2 E(\tau_0)}&\ge \frac{|S|-L_{low}^{-1}\ln(2)}{3+2 L_{low}^{-1}}
\end{align*}
we get exactly the constant \(C_{uni}\).
\end{proof}

\section{MSE of the estimated threshold}
\label{regret_bounds}
We use the setting of Section \ref{sec 2}. The goal is to use the results from Section \ref{sec_abstract} to derive an upper bound on the MSE for the threshold estimator resulting from the ILU algorithm.

\begin{lemma}
\label{Var_schranke}
We have \(\Var(\tau_0)<\infty\).
\end{lemma}
\begin{proof}
Since $\lambda \in \cM(L)$, we have \(\lambda(u) \ge L_{low}\)  for all \(u \in [S,\infty)\). 
Moreover, 
\begin{align*}
E(\tau_0^2) &= \int_0^\infty P(\tau_0^2 >t) dt = \int_0^\infty P(\tau_0 >\sqrt{t})dt \\
&=2 \int_0^\infty P(\tau_0 >u)u du \\
&= 2 \int_0^\infty e^{-\int_0^u \lambda(y)dy}u du \\
&\le 2 \int_0^\infty e^{-L_{low} u}u du \\
&=2 \left([-\frac{1}{L_{low}} e^{-L_{low} u}u]_0^\infty + \frac{1}{L_{low}} \int_0^\infty e^{-L_{low}u}\, du\right) \\
&= \frac{2}{L_{low}^2} < \infty.
\end{align*} 
It is obvious that \(|E(\tau_0)| < \infty\). Therefore \(\Var(\tau_0)<\infty\).
\end{proof}
All statements concerning the estimators in an arbitrary epoch \(l \in \mathbb{N}_0\) derived from the ILU algorithm in this section are to be understood conditional on \(\cA_{l-1}\). Under the corresponding \(\sigma\)-algebra, \(b = \hat b_{l-1}\) is deterministic, and the rounds within an epoch are independent and identically distributed. Indeed, \(\hat b_{l-1}\) is \(\cA_{l-1}\)-measurable, and the processes \(N^k, \, k \in \cT_{l-1},\) are independent of \(\cA_{l-1}\). Hence, conditionally on \(\cA_{l-1}\), the processes \(N^k, \, k \in \cT_{l-1},\) are independent Poisson processes with intensity \(\lambda\), and all statements below hold under the regular conditional distribution \(P_\lambda (\cdot | \cA_{l-1})\) with \(b=\hat b_{l-1}\) treated as a constant. Therefore, we omit the corresponding condition in the notation for now.

\begin{lemma}[Properties of \(\hat G\)]
\label{PropertiesEstimatorG}
   For an arbitrary epoch \(l\in \mathbb{N}\) with \(m:=|\cT_{l-1}|, b:=\hat b_{l-1}\) and \(y\in[b,0]\), the following holds 
\begin{enumerate}[label=(\arabic*)]
    \item \label{prop_1_est_surv}
    \(
        m\,\hat G(y)=\sum_{i\in \mathcal T_{l-1}}\mathbf{1}_{\{\tau^i_b>y\}}\ \sim\ \mathrm{Bin}\big(m,G(y)\big).
    \)
    \item \label{prop_2_est_surv} \(\hat G(b)=1\) and \(\hat G(y) \ge \hat G(0).\)
    \item \label{prop_3_est_surv} 
    \(E_\lambda(\hat G(y))=G(y).\)
    \item \label{prop_4_est_surv} \(
       E_\lambda[(\hat G(y)-G(y))^2] \le \frac1m G(y)(1-G(y))\le \frac{1}{4m}.
   \)
\end{enumerate}
\end{lemma}
\begin{proof}
Properties \ref{prop_1_est_surv}-\ref{prop_3_est_surv} are straightforward to show.

From property \ref{prop_3_est_surv} follows \begin{align*}
    E_\lambda[(\hat G(y)-G(y))^2] &= \Var (\hat{G}(y)).
\end{align*}
With property \ref{prop_1_est_surv} we get
\begin{align*}
    \Var (\hat{G}(y))=\frac1m G(y)(1-G(y))\le \frac{1}{4m}.
\end{align*}
\end{proof}

\begin{lemma}[Uniform bound MSE of \(\hat G\)]
\label{uniform_bound_g}
    For an arbitrary epoch \(l\in \mathbb{N}\) with \(m:=|\cT_{l-1}|, b:=\hat b_{l-1}\) and \(y\in[b,0]\), the following holds
    \begin{align*}
        E_\lambda[\sup_{y\in[b,0]}(\hat G(y)-G(y))^2] \le \frac1m
    \end{align*}
\end{lemma}
\begin{proof}
    We know that the sequence \((\tau_b^i)_{i\in\{2^{l-1}-1,\ 2^{l-1},\ \dots,\ 2^{l}-2\}}\) is i.i.d. The estimator \(\hat G\) is the empirical survival function of this sequence of random variables. By using the Dvoretzky–Kiefer–Wolfowitz inequality we directly get 
    \begin{align*}
        P_\lambda(\sup_{y\in[b,0]}|\hat G(y)-G(y)|>\epsilon)\le 2e^{-2m\epsilon^2}
    \end{align*}
    for all \(\epsilon >0.\)
    It follows
    \begin{align*}
        E_\lambda[\sup_{y\in[b,0]}(\hat G(y)-G(y))^2] &= \int_0^\infty P_\lambda(\sup_{y\in[b,0]}(\hat G(y)-G(y))^2 > t) \, dt\\
        &= \int_0^\infty P_\lambda(\sup_{y\in[b,0]}|\hat G(y)-G(y)| > \sqrt{t}) \, dt \\
        &\le \int_0^\infty 2e^{-2mt} \, dt\\
        &=2[\frac{-1}{2m}e^{-2mt}]_0^\infty=\frac1m.
    \end{align*}
    
\end{proof}

\begin{lemma}[Properties of \(\hat I\)]
\label{Properties_I}
    For an arbitrary epoch \(l\in \mathbb{N}\) with \(m:=|\cT_{l-1}|, b:=\hat b_{l-1}\) and \(y\in[S,b]\), the following holds 
\begin{enumerate}[label=(\arabic*)]
    \item \label{prop_1_est_integrated}
    \(
    m\,\hat I(y)=\sum_{i\in\mathcal T_{l-1}}\big(N^i_b-N^i_y\big)\ \sim\ \mathrm{Poi}\big(m\,[\Lambda(y)-\Lambda(b)]\big). 
   \)
    \item \label{prop_2_est_integrated} \(E_\lambda[\hat I (y)]=[\Lambda(y)-\Lambda(b)]\le L|S|.\)
    \item \label{prop_3_est_integrated} $\MSE_\lambda[\hat I (y)]=\dfrac{[\Lambda(y)-\Lambda(b)]}{m}\le\dfrac{L|S|}{m}$.
    \end{enumerate}
\end{lemma}
\begin{proof}
    Straightforward.
\end{proof}

\begin{lemma}[Uniform bound MSE of \(\hat I\)]
\label{uniform_mse_I}
    For an arbitrary epoch \(l\in \mathbb{N}\) with \(m:=|\cT_{l-1}|, b:=\hat b_{l-1}\) and \(y\in[S,b]\), the following holds
    \begin{align*}
        E_\lambda[\sup_{y\in[S,b]} [\hat I(y)-(\Lambda(y)-\Lambda(b))]^2]\le \frac{4L|S|}{m}.
    \end{align*}
\end{lemma}
\begin{proof}
Define \(I(y):=\Lambda(y)-\Lambda(b).\)
    Let \(X_t=\hat I(b-t) - I(b-t)\) for \(t \in [0,b-S].\) It is straightforward to show that \((X_t)\) is a square-integrable martingale w.r.t. the natural filtration and therefore Doob's \(L^2-\)maximum inequality is applicable. Together with property \ref{prop_3_est_integrated} from Lemma \ref{Properties_I}, we get
    \begin{align*}
        E_\lambda[\sup_{t \in [0,b-S]}X_t^2] \le 4E_\lambda[X_{b-S}^2] = 4 E_\lambda[(\hat I(S) - I(S))^2]\le 4 L|S|/m.
    \end{align*}
\end{proof}
    
\begin{lemma}
\label{structure_lambda}
    For an arbitrary epoch \(l\in \mathbb{N}\) with \(m:=|\cT_{l-1}|, b:=\hat b_{l-1}\) and \(y\in[S,b]\) holds that 
    \begin{enumerate}[label=(\arabic*)]
        \item \label{prop_lambda_1}
        \(\hat \Lambda_l(0)=0\) and \(0 \le \hat \Lambda_l \le L|S|\) and \(\hat \Lambda\) is non-increasing on \([S,0]\). In particular, \(\hat \Lambda_l\) is measurable and bounded.
        \item \label{prop_lambda_2}
        \begin{flalign}
        \label{kontraktion_lambda_est}
            &|\hat \Lambda_l(y)-\Lambda(y)|\le |\hat \Lambda_l^{uc}(y)-\Lambda(y)| \, \text{ for all } y \in [S,0]. &&
        \end{flalign}
    \end{enumerate}
\end{lemma}
\begin{proof}
    Property \ref{prop_lambda_1}: Non-negativity is straightforward. We have \(\hat G(y) \ge \hat G(0)\) for \(b<y <0 \) and \(\hat G(0)\le 1\), therefore \(\hat\Lambda_l^{uc}\) is always non-negative. 
    The monotonicity follows directly from the fact that \(\hat G\) and \(\hat I\) are non-increasing.  
    The rest is straightforward. \\
    Property \ref{prop_lambda_2}:
    Case 1: \(\hat \Lambda_l(y) \ge \Lambda(y):\)
    \[
        |\hat \Lambda_l(y)-\Lambda(y)|=\hat \Lambda_l(y)-\Lambda(y)
        \le \hat \Lambda_l^{uc}(y)-\Lambda(y).
    \]
    Case 2: \(\hat \Lambda_l(y) < \Lambda(y).\) We know \( \Lambda(y) \le L|S|.\) Therefore, \(\hat \Lambda_l(y)=\hat\Lambda_l^{uc}(y)\).

\end{proof}

\begin{lemma}
\label{prob_Ereignis_E}
    For an arbitrary epoch \(l\in \mathbb{N}\) with \(m:=|\cT_{l-1}|, b:=\hat b_{l-1}\) define \(E:=\{\hat G(0) \ge \frac12 G(0)\}.\) Then
    \begin{align}
    \label{probE}
    P_\lambda(E^c)\le \frac{4}{m}e^{L|S|}.
    \end{align}
\end{lemma}
\begin{proof}
    First, we apply Tschebycheff inequality and use Lemma \ref{PropertiesEstimatorG}
    \begin{align*}
        P_\lambda(E^c)&=P_\lambda\left(\hat G(0) < \frac12 G(0)\right)=P_\lambda\left(\hat G(0)- G(0) < -\frac12 G(0)\right)\\
        &\le P_\lambda\left(|\hat G(0)- G(0)| \ge \frac12 G(0)\right)\\
        &\le \frac{4G(0)(1-G(0))}{mG(0)^2}\le \frac{4}{mG(0)}.
    \end{align*}
    With \(G(0)=e^{-\Lambda(b)}\ge e^{-L|S|}\) the statement follows.
\end{proof}

\begin{proposition}

\label{ordnungLambda}
 For all epochs \(l\ge1\) and \(y \in [S,0]\) with \(m=|\cT_{l-1}|\), the following holds
    \begin{align}
    E_\lambda[(\hat\Lambda_l(y)-\Lambda(y))^2]\le c_2/m
    \end{align}
    and
    \begin{align}
         E_\lambda[\sup_{y\in[S,0]}(\hat\Lambda_l(y)-\Lambda(y))^2]\le c_1/m
    \end{align}
    where \(c_2=4e^{2L|S|}+2L|S|+4L^2S^2e^{L|S|},\quad c_1=32e^{2L|S|}+8L|S|+4L^2S^2e^{L|S|}.\)
    
\end{proposition}
\begin{proof}
For notational convenience, we set \(\hat \Lambda^{uc}:=\hat \Lambda_l^{uc}\) and \(\hat \Lambda:=\hat \Lambda_l\).
    Define the event \(E\) as in Lemma \ref{prob_Ereignis_E}. \\
    First, we working under event \(E\). Note that on \(E\) it holds \(\min\{\hat G(y), G(y)\}\ge\min\{\hat G(0), G(0)\}\ge \frac12 G(0) \ge \frac12 e^{-L|S|}.\) 
    \\
    Note, because of the mean value theorem the following holds
    \begin{align}
    \label{abschätzung_unter_e}
        |\ln\hat G(y)-\ln G(y)|\le \frac{|\hat G(y)-G(y)|}{\min\{\hat G(y), G(y)\}} \le 2e^{L|S|} |\hat G(y)-G(y)|.
    \end{align}
    For \(b < y < 0\) we can represent \begin{align*}
        \hat \Lambda^{uc}(y)-\Lambda(y) &= \ln \hat G(y) - \ln \hat G(0) - (\ln G(y)-\ln G(0)) \\ &=\ln \hat G(y) - \ln G(y) - (\ln \hat G(0)-\ln G(0)).
    \end{align*}
    With \((a+b)^2\le 2a^2+2b^2\) and \eqref{abschätzung_unter_e}  we get
    \begin{align*}
        E_\lambda\left[(\hat \Lambda^{uc}(y)-\Lambda(y))^2\mathbf{1}_E\right] &\le E_\lambda\left[2((\ln \hat G(y) - \ln G(y))^2+(\ln \hat G(0)-\ln G(0))^2)\mathbf{1}_E\right]\\
        &\le E_\lambda[8e^{2L|S|}((\hat G(y)-G(y))^2+(\hat G(0)-G(0))^2)] \\
        &\le 8e^{2L|S|} \left(\frac{1}{4m}+\frac{1}{4m}\right)=4 e^{2L|S|}/m.
    \end{align*}
    The last inequality used Lemma \ref{PropertiesEstimatorG}.\\
 For \(S\le y \le b\) define again \(I(y)=\Lambda(y)-\Lambda(b)\). We can represent 
\begin{align*}
        \hat \Lambda^{uc}(y)-\Lambda(y) &= -\ln \hat G(0) +\hat I(y) - (-\ln G(0) + I(y)) \\ &=- (\ln \hat G(0)-\ln G(0)) + (\hat I(y)-I(y)).
    \end{align*}
With analogous procedure as for the previous case and using Lemma \ref{PropertiesEstimatorG} and \ref{Properties_I} we get
\begin{align*}
        E_\lambda\left[(\hat \Lambda^{uc}(y)-\Lambda(y))^2\mathbf{1}_E\right] &\le E_\lambda\left[2((\ln \hat G(0) - \ln G(0))^2+( \hat I(y)- I(y))^2)\mathbf{1}_E\right]\\
        &\le 8e^{2L|S|} \frac{1}{4m}+2 L|S|/m=\frac1m(  2e^{2L|S|}+2L|S|).
\end{align*}
Therefore
\begin{align*}
        E_\lambda\left[(\hat \Lambda^{uc}(y)-\Lambda(y))^2\mathbf{1}_E\right] &\le \frac1m\max\{4e^{2L|S|}, 2e^{2L|S|}+2L|S|\} \le \frac1m(4e^{2L|S|}+2L|S|).
\end{align*}
Next, we derive a uniform bound for \((\hat \Lambda^{uc}-\Lambda)^2\) on the event \(E\).
First, with Inequality \eqref{abschätzung_unter_e}, we get
\begin{align*}
    \sup_{y\in(b,0)} |\hat \Lambda^{uc}(y)-\Lambda(y)| &= \sup_{y\in(b,0)} |\ln \hat G(y) - \ln G(y) - (\ln \hat G(0)-\ln G(0))|\\
    &\le 2\sup_{y\in(b,0)} e^{L|S|} (|\hat G(y) -G(y)| + |\hat G(0) -G(0)|)\\
    &\le 4 e^{L|S|} \sup_{y \in [b,0]} |\hat G(y) -G(y)|.
\end{align*}
Now
\begin{align*}
    \sup_{y\in[S,b]} |\hat \Lambda^{uc}(y)-\Lambda(y)| &= \sup_{y\in[S,b]} |- (\ln \hat G(0)-\ln G(0)) + (\hat I(y)-I(y))|\\
    &\le 2e^{L|S|} |\hat G(0)-G(0)| +\sup_{y\in[S,b]} |\hat I(y)-I(y)|.
\end{align*}
Therefore
\begin{align*}
    \sup_{y\in[S,0]} |\hat \Lambda^{uc}(y)-\Lambda(y)| &\le \max\{4 e^{L|S|} \sup_{y \in [b,0]} |\hat G(y) -G(y)|,2e^{L|S|} |\hat G(0)-G(0)| + \sup_{y\in[S,b]}|\hat I(y)-I(y)| \}\\ 
    &\le 4 e^{L|S|} \sup_{y \in [b,0]} |\hat G(y) -G(y)|+ \sup_{y\in[S,b]}|\hat I(y)-I(y)|.
    \end{align*}
  
It follows
\begin{align}
\label{estimate_lambda_e}
    E_\lambda[\sup_{y\in[S,0]}(\hat \Lambda^{uc}(y)-\Lambda(y))^2 \mathbf{1}_E]&\le E_\lambda[(4 e^{L|S|} \sup_{y \in [b,0]} |\hat G(y) -G(y)|+ \sup_{y\in[S,b]}|\hat I(y)-I(y)|)^2\mathbf{1}_E] \notag \\
    &\le 32e^{2L|S|}E_\lambda [\sup_{y \in [b,0]} (\hat G(y) -G(y))^2 \mathbf{1}_E]+2 E_\lambda[\sup_{y\in[S,b]}(\hat I(y)-I(y))^2 \mathbf{1}_E] \notag\\
    &\le \frac1m(32e^{2L|S|}+8L|S|).
\end{align}
For the last inequality we used Lemma \ref{uniform_bound_g} and Lemma \ref{uniform_mse_I} and got the desired uniform bound on \(E\).
Using Lemma \ref{structure_lambda} and Lemma \ref{prob_Ereignis_E} we achieve
\begin{align}
\label{estimate_lambda_ec}
    E_\lambda[\sup_{y\in[S,0]}(\hat \Lambda(y)-\Lambda(y))^2 \mathbf{1}_{E^c}]\le \frac1m(L^2S^24e^{L|S|}).
\end{align}
Therefore, by using the previously derived estimates \eqref{estimate_lambda_e} and \eqref{estimate_lambda_ec} as well as \eqref{kontraktion_lambda_est}  we get for the uniform bound of \((\hat \Lambda-\Lambda)^2\) 
\begin{align*}
E_\lambda[\sup_{y\in[S,0]}(\hat \Lambda(y)-\Lambda(y))^2]&\le E_\lambda[\sup_{y\in[S,0]}(\hat \Lambda^{uc}(y)-\Lambda(y))^2\mathbf{1}_E]+E_\lambda[\sup_{y\in[S,0]}(\hat \Lambda(y)-\Lambda(y))^2\mathbf{1}_{E^c}]\\
    &\le \frac1m(32e^{2L|S|}+8L|S|+L^2S^24e^{L|S|}).
\end{align*}
\end{proof}

\begin{lemma}
\label{MSE_phi_K_ge_1}
    For an arbitrary epoch \(l\in \mathbb{N}\) with \(m:=|\cT_{l-1}|\) and \(b:=\hat b_{l-1}\) holds that
    \begin{align}
    E_\lambda[(\hat \phi_l - E_\lambda(\tau_0))^2 \mathbf{1}_{\{K\ge 1\}}]\le \Var(\tau_0)E_\lambda\left[\frac1K \mathbf{1}_{\{K\ge1\}}\right]\le 2\Var(\tau_0)e^{L|S|}/m,      \end{align}
    where \(K\) denotes the number of rounds from the previous epoch \(l-1\) where the first free parking lot after 0 was observed.
\end{lemma}
\begin{proof}
Let \((\nu_i)_{i\in\cT_{l-1}}\) be a sequence of random variables where \(\nu_j=1, \,\,  j \in \cT_{l-1}\) if in round \(j\) the first parking lot after 0 is observed, and  \(\nu_j=0\) else. Therefore, within epoch \(l-1\), it holds that the sequence is i.i.d. and \(\nu_j\sim Ber(G(0))\). Because the increments of a Poisson process are independent and therefore \(\sigma(N^j_t: t \in [b,0])\) and \(\sigma(N^j_t: t \in (0,\infty))\) are independent it follows that \(\nu_j\) and \(\tau_0^j\) are independent as well. Because all rounds are independent from each other we know that \((\nu_i)_{i\in\cT_{l-1}}\) and \((\tau_0^i)_{i\in\cT_{l-1}}\) are independent. 

Let \(x \in \{0,1\}^{|\cT_{l-1}|}\) and \(J(x):=\{i:x_i=1\}\). On the event \(\{\nu=x\}\) the set \(J(x)\) is deterministic and we know that the conditional distribution of \((\tau_0^i)\) under \(\{\nu=x\}\) is equal to the unconditioned distribution because \((\tau_0^i)_{i\in\cT_{l-1}}\) are independent. We have
\begin{align*}
    E_\lambda\left[\left(\frac{1}{|J(x)|} \sum_{i\in J(x)}\tau_0^i - E_\lambda(\tau_0)\right)^2\mid \nu=x\right]&=E_\lambda\left[\left(\frac{1}{|J(x)|} \sum_{i\in J(x)}\tau_0^i - E_\lambda(\tau_0)\right)^2\right]\\
    &= \Var_\lambda\left(\frac{1}{|J(x)|} \sum_{i\in J(x)}\tau_0^i\right)\\
    &=\frac{1}{|J(x)|}\Var(\tau_0).
\end{align*}
Define \(\hat\eta_l:=\frac{1}{K} \sum_{i \in \cT_{l-1}} \tau^{i}_0 \, \mathbf{1}_{\{\tau_b^i >0\}}\). The estimator \(\hat \eta_l\) is exactly the unclipped version of \(\hat \phi_l\).
With the tower property follows
\begin{align*}
    E_\lambda\left[(\hat\eta_l-E_\lambda(\tau_0))^2\mathbf{1}_{\{K\ge1\}}\right] &=E_\lambda\left[E_\lambda\left[(\hat \eta_l-E_\lambda(\tau_0))^2 \bigg|\nu\right]\mathbf{1}_{\{K\ge1\}}\right]\\
    &=\Var(\tau_0)E_\lambda\left[\frac1K \mathbf{1}_{\{K\ge1\}}\right].
\end{align*}
Lastly, we know that \(E_\lambda(\tau_0)\in[\frac1L,\frac{1}{L_{low}}].\) 

\textit{Case 1: }\(\hat \eta_l \in[\frac1L,\frac{1}{L_{low}}].\) Then \(\hat \phi_l = \hat \eta_l\) and therefore \(|\hat \phi _l - E_\lambda(\tau_0)|\le|\hat\eta_l - E_\lambda(\tau_0)|. \) 

\textit{Case 2: }\(\hat \eta_l < \frac1L.\) Then \(\frac1L=\hat \phi_l \ge \hat \eta_l\ge0\) and 
\begin{align*}
    |\hat \phi _l - E_\lambda(\tau_0)|=E_\lambda(\tau_0)- \hat \phi_l \le E_\lambda(\tau_0)- \hat \eta_l
    =|\hat\eta_l - E_\lambda(\tau_0)|.
\end{align*}

\textit{Case 3: } \(\hat \eta_l > \frac{1}{L_{low}}.\) Then \(\frac{1}{L_{low}}=\hat \phi_l \le \hat \eta_l\) and 
\begin{align*}
    |\hat \phi _l - E_\lambda(\tau_0)|=\hat \phi_l- E_\lambda(\tau_0) \le \hat \eta_l-E_\lambda(\tau_0)
    =|\hat\eta_l - E_\lambda(\tau_0)|.
\end{align*}
It follows that \(|\hat \phi _l - E_\lambda(\tau_0)|\le|\hat\eta_l - E_\lambda(\tau_0)| \) for every path. Therefore 
\begin{align*}
    E_\lambda[(\hat \phi_l - E_\lambda(\tau_0))^2 \mathbf{1}_{\{K\ge 1\}}]\le  E_\lambda[(\hat \eta_l - E_\lambda(\tau_0))^2 \mathbf{1}_{\{K\ge 1\}}] = \Var(\tau_0)E_\lambda\left[\frac1K \mathbf{1}_{\{K\ge1\}}\right].
\end{align*}
In the final step we have to calculate \(E_\lambda[\frac1K \mathbf{1}_{\{K\ge1\}}]\). From Lemma \ref{PropertiesEstimatorG} we know that \(K \sim Bin(m, p)\) with \(p:=G(0).\) By using the identity \(\frac{m+1}{k+1}\binom{m}{k}=\binom{m+1}{k+1}\) we get 
\begin{align*}
    E_\lambda\left[\frac{1}{1+K}\right]&=\sum_{k=0}^m \frac{1}{1+k}\binom{m}{k}p^k(1-p)^{m-k} \\
    &= \frac{1}{1+m}\frac1p\sum_{k=0}^m \binom{m+1}{k+1}p^{k+1}(1-p)^{(m+1)-(k+1)}\\
    &=\frac{1}{1+m}\frac1p\sum_{j=1}^{m+1} \binom{m+1}{j}p^{j}(1-p)^{(m+1)-j}\\
    &=\frac{1}{1+m} \frac1p \left(1-(1-p)^{m+1}\right)\le \frac{1}{m} \frac1p.
\end{align*}
With the simple inequality \(\frac1k \le \frac{2}{k+1}\) for \(k\ge1\) we have
\begin{align*}
   E_\lambda\left[\frac1K \mathbf{1}_{\{K\ge1\}}\right]\le 2 \frac{1}{mp} 
\end{align*}
With \(p=G(0)\ge e^{-L|S|}\) the statement follows.
 
\end{proof}

\begin{lemma}
\label{MSE_phi_K_0}
    For an arbitrary epoch \(l\in \mathbb{N}\) with \(m:=|\cT_{l-1}|\) and \(b:=\hat b_{l-1}\), it holds that
    \begin{align}
    E_\lambda\left[(\hat \phi_l - E_\lambda(\tau_0))^2 \mathbf{1}_{\{K=0\}}\right]\le \left(\frac1L-\frac{1}{L_{low}}\right)^2 e^{L|S|-1}/m ,      \end{align}
    where \(K\) denotes the number of rounds from the previous epoch \(l-1\) where the first free parking lot after 0 was observed.
\end{lemma}
\begin{proof}
    On \(\{K=0\}\) we have \(\hat \phi_l=\frac1L.\) Moreover, \(E_\lambda(\tau_0)\in[\frac1L,\frac{1}{L_{low}}]\). Hence, \((\hat\phi_l-E_\lambda(\tau_0))^2\le \left(\frac1L-\frac{1}{L_{low}}\right)^2\) on \(\{K=0\}\).
    Due to Lemma \ref{PropertiesEstimatorG} we know \(K\sim Bin(m,G(0))\). So
    \begin{align*}
        P_\lambda(K=0)=(1-G(0))^m\le (1-e^{-L|S|})^m\le e^{-\exp(-L|S|)m}.
    \end{align*}
    The function \(x\to xe^{-\exp(-L|S|)x}\) on \([0,\infty)\) attains its maximum at \(\exp(L|S|)\). Therefore,
    \begin{align*}
        e^{-\exp(-L|S|)m}= \frac1m m e^{-\exp(-L|S|)m} \le \frac1m \sup_m m e^{-\exp(-L|S|)m}  \le \frac1m e^{L|S|}e^{-1}.
    \end{align*}
    
\end{proof}
\begin{proposition}
\label{ordnung_phi_est}
    For all epochs \(l\ge1\) with \(m=|\cT_{l-1}|\), the following holds
    \begin{align}
        E_\lambda[(\hat \phi _l-E_\lambda(\tau_0))^2]\le c_3/m,
    \end{align}
    where \(c_3:=5e^{L|S|}L_{low}^{-2}\).
\end{proposition}
\begin{proof}
    By Lemma \ref{MSE_phi_K_ge_1} we know \(E_\lambda[(\hat \phi _l-E_\lambda(\tau_0))^2\mathbf{1}_{\{K\ge1\}}] \le 2\Var(\tau_0)e^{L|S|}/|\cT_{l-1}|\). By the proof of Lemma \ref{Var_schranke} we have \(\Var(\tau_0)\le 2L_{low}^{-2}.\)
Using Lemma \ref{MSE_phi_K_0} we get \(E_\lambda[(\hat \phi _l-E_\lambda(\tau_0))^2\mathbf{1}_{\{K=0\}}] \le\left(\frac1L-\frac{1}{L_{low}}\right)^2 e^{L|S|-1}/|\cT_{l-1}|\le L_{low}^{-2} e^{L|S|}/|\cT_{l-1}|\).

    With \(E_\lambda[(\hat \phi _l-E_\lambda(\tau_0))^2]=E_\lambda[(\hat \phi _l-E_\lambda(\tau_0))^2\mathbf{1}_{\{K\ge1\}}]+E_\lambda[(\hat \phi _l-E_\lambda(\tau_0))^2\mathbf{1}_{\{K=0\}}]\) the statement follows.
\end{proof}

\section{Proof of Theorem \ref{main_res_1}}\label{sec proof main res 1}
We need to show that there exists a constant $C$ such that for all $\lambda \in \cM(L)$ and $T \in \N$ we have $\cR^{ILU(L)}_\lambda(T) \le C \ln(T+2)$.  

Let \(l \ge1\) and \(m_{l-1}=|\cT_{l-1}|=2^{l-1}.\) Due to Lemma \ref{structure_lambda} we know that \(\hat \Lambda_l:[S,0]\to [0,\infty)\) is measurable, bounded and non-increasing and \(\hat \Lambda_l(0) =0\). Moreover, by Proposition \ref{ordnungLambda} and \ref{ordnung_phi_est}, we know that 
\begin{align}
         E_\lambda[\sup_{y\in[S,0]}(\hat\Lambda_l(y)-\Lambda(y))^2|\cA_{l-1}]\le c_1/m_{l-1}
    \end{align}
    and 
\begin{align}
    E_\lambda[(\hat\Lambda_l(y)-\Lambda(y))^2|\cA_{l-1}]\le c_2/m_{l-1}
    \end{align}
    and 
\begin{align}
        E_\lambda[(\hat \phi _l-E_\lambda(\tau_0))^2|\cA_{l-1}]\le c_3/m_{l-1},
    \end{align}
    where \(c_1=32e^{2L|S|}+8L|S|+4L^2S^2e^{L|S|}; \quad c_2=4e^{2L|S|}+2L|S|+4L^2S^2e^{L|S|}; \quad c_3:=5e^{L|S|}L_{low}^{-2}.\)
Note that the constants \(c_1,c_2\) and \(c_3\) only depend on the environment parameter \(L\) and \(S\).\\
Now, we can apply Proposition \ref{abstract_mse_b} and Corollary \ref{cor_uniform_constant} conditioned on \(\cA_{l-1}\) and \(n=m_{l-1}\)
\begin{align}
\label{final_MSE}
    E_\lambda[(\hat b_l - b^*)^2]=E_\lambda[E_\lambda[(\hat b_l - b^*)^2|\cA_{l-1}]] \le E_\lambda[C_{uni}/m_{l-1}]=C_{uni}/2^{l-1},
\end{align}
where \(C_{uni}\) as in Corollary \ref{cor_uniform_constant}.
Note that \(C_{uni}\) only depends on the environment parameters \(L\) and \(S\).\\
Now we want to establish the relation between the optimality gap and the squared error of a threshold with respect to the optimal threshold.
To do this, we use that $\lambda$ belongs to the environment class $\cM(L)$. This implies that $\lambda$ is continuously differentiable and hence $\Delta_\lambda$ twice continuously differentiable (see Lemma \ref{suff condi Delta smooth}). Therefore, using $\Delta_\lambda(b^*) = \Delta_\lambda'(b^*) = 0$, Taylor's theorem implies 
\begin{align}
\label{opti_gap_smooth}
\Delta_\lambda(b) = \frac12 \Delta_\lambda''(\theta) (b- b^*)^2  
\end{align}
for some $\theta$ between $b$ and $b^*$. \\
Moreover, we define $c:= \sup_{\lambda \in \cM(L)} \sup_{b \in [S, 0]} \Delta_\lambda''(b)$. Since $\lambda \in \cM(L)$ one can derive from Equation \eqref{sec deri Delta} that $c$ depends on $S$ and $L$ only, and that $c < \infty$.\\
With \eqref{final_MSE} we get for \(l\ge1\)
\begin{align}
\label{roundwise_regret_smooth}
\sup_{\lambda \in \cM(L)} E_\lambda[ \Delta_\lambda(\hat b_l)] \le \frac12 c \sup_{\lambda \in \cM(L)} E_\lambda[ (\hat b_l - b^*)^2]\le c C_{uni} /2^l.
\end{align}
For \(l=0\) we have 
\begin{align*}
    \sum_{n \in \cT_0} E_\lambda[ \Delta_\lambda(\hat b_0)] = \Delta_\lambda(0)\le \frac{c}{2}(0-b^*)^2\le cS^2/2,
\end{align*}
where the last inequality comes from the Taylor expansion.\\
Let $(\pi_n)$ be the sequence of thresholds chosen by the ILU algorithm.\\
Each of the \(2^l\) rounds of an arbitrary epoch \(l\ge1\) uses the same threshold \(\hat b_l\)  and therefore we achieve with \eqref{roundwise_regret_smooth}
\begin{align*}
    \sum_{n\in\cT_l} E_\lambda[\Delta_\lambda(\pi_n)] = \sum_{n\in\cT_l} E_\lambda[\Delta_\lambda(\hat b _l)] \le 2^l \frac{c C_{uni}}{2^l}=cC_{uni}.
\end{align*}
If we have \(T\in \mathbb{N}\) rounds, we have \(\min\{n\in\mathbb{N}: 2^n-2\ge T\}=\lceil\log_2(T+2)\rceil\) many epochs. The last epoch may be incomplete.\\
Finally, we get for \(T\ge 0\)
\begin{align*}
    \cR_\lambda^{ILU(L)}(T)&=\sum_{n=0}^T E_\lambda[\Delta_\lambda(\pi_n)]\le  \sum_{l=0}^{\lceil\log_2(T+2)\rceil-1} \sum_{n \in \cT_l} E_\lambda[\Delta_\lambda(\hat b_l)]\\
    &\le cS^2/2+ \sum_{l=1}^{\lceil\log_2(T+2)\rceil-1} cC_{uni}\\
    &\le cS^2/2+cC_{uni}\lceil\log_2(T+2)\rceil\\
    &\le cS^2/2+cC_{uni}\frac{\ln(T+2)}{\ln2}+cC_{uni}\\
    &\le cS^2/2\frac{\ln(T+2)}{\ln2}+cC_{uni}\frac{\ln(T+2)}{\ln2}+cC_{uni} \frac{\ln(T+2)}{\ln2}\\
    &\le \left[\frac{c}{\ln2}(S^2+2C_{uni})\right]\ln(T+2)
\end{align*}
The constant \(\frac{c}{\ln2}(S^2+2C_{uni})\) only depends on the environment parameters \(S\) and \(L\) and therefore Theorem \ref{main_res_1} is proven.
\section{Lower bound: Proof of Theorem~\ref{main_res_3}}
\label{lowerbound}
In the following we prove Theorem~\ref{main_res_3}. We need to show that the minimax regret 
\[ \inf_{\pi \in \Pi} \sup_{\lambda \in \mathcal{M}(L)} \cR^\pi_\lambda(T)\] 
grows at least logarithmically. The crux of the argument is that it is enough to derive a logarithmic lower bound for \(\mathcal{H} (L):=\{ \lambda \in \mathcal{M}(L) : \lambda \text{ constant}\}\), the subclass of constant intensity functions. Indeed, note that  
\begin{align*}
\inf_{\pi \in \Pi} \sup_{\lambda \in \mathcal{M}(L)} \cR^\pi_\lambda(T) \ge  \inf_{\pi \in \Pi} \sup_{\lambda \in \mathcal{H}(L)} \cR^\pi_\lambda(T).
\end{align*}
Therefore, a logarithmic lower bound on $\cH(L)$ implies a lower bound for the whole environment class \(\mathcal{M}(L)\). 

Note that if the unknown intensity function is from the class \(\mathcal{H}(L)\), then the stopping agent observes a homogeneous Poisson process in each round. The estimation of the intensity function is thus simply a one parameter estimation, namely the estimation of \(\lambda \in [a,b],\) where \(a:=\frac{\ln(2)}{|S|}+\frac1L\) and $b:= L$. 

In addition, we assume that we have full information in every round \(k\in \mathbb{N}_0\), i.e.\ that we can observe the homogeneous Poisson process on \([S,\tau_0^k]\) in every round. In other words, we observe the process up to the first jump after 0, irrespective of the chosen threshold.
We denote by \(\Pi^{full}\) the set of all sequences \((\pi_n)_{n\ge 0}\) of \([S,0]\)-valued random variables such that \(\pi_n\) is \(\cF_n^{full}\)-measurable, where \(\cF_0^{full}=\{\emptyset, \Omega\}\) and \(\cF_n^{full}=\bigvee_{k=0}^{n-1}\sigma(N_t^k: t \in [S,\tau_0^k])\) for \(n\ge 1\). It holds \(\Pi\subset \Pi^{full}.\) 
Since the class of policies with full information is larger than the class of policies with partial information, we have 
\begin{align}\label{esti against fi policies}
\inf_{\pi \in \Pi} \sup_{\lambda \in \mathcal{H}(L)} \cR^\pi_\lambda(T) \ge \inf_{\pi \in \Pi^{full} } \sup_{\lambda \in \mathcal{H}(L)} \cR^\pi_\lambda(T).    
\end{align}
Moreover, to simplify the following analysis we introduce the class of so-called cut-off full information policies \[\Pi^{cut}:= \left\{\pi^{cut} = (\pi_n^{cut})_{n\ge 0}| \exists \pi^{full} \in \Pi^{full} \text{ s.t. } \forall n \ge 0 : \pi_n^{cut}=\min\{\pi_n^{full}, -\frac{\ln(2)}{L}\}\right\}.\]  The idea is that for \(\lambda \in \mathcal{H}(L)\) we know that \(b^*\in (S, -\frac{\ln(2)}{L}] \) by Corollary \ref{constant_lambda_cost}. Therefore, cutting off every policy at \(-\frac{\ln(2)}{L}\) cannot worsen the regret. Let \(\pi \in \Pi^{full}\) and \(\pi^{cut}\) the corresponding cut-off policy. Fix \(n \in \mathbb{N}_0.\) \\
\textbf{case 1:}  We have \(\pi_n < -\frac{\ln(2)}{L}\). Therefore
\begin{align*}
    \Delta(\pi_n)= E|\tau_{\pi_n}| - E|\tau_{b^*}| = E|\tau_{\pi_n^{cut}}| - E|\tau_{b^*}| =  \Delta(\pi_n^{cut})
\end{align*}
\textbf{case 2:}  We have \(\pi_n \in [L^*,0],\) where \( L^* :=-\frac{\ln(2)}{L}\). By Corollary \ref{constant_lambda_cost} we know the cost function for \(\lambda \in \mathcal{H}(L)\) is \(\Delta_\lambda(b) = \frac{1}{\lambda}(2 e^{\lambda b} -1) -b-\frac{\ln(2)}{\lambda}\) for \(b \in [S,0].\) We want to show that \(\Delta_\lambda(x) \ge \Delta_\lambda(L^*)\) for all \( x \in [L^*,0].\) We have \(\Delta'_\lambda(x)=2e^{\lambda x}-1.\) Recall that \(\lambda \in [\ln(2)/|S|+\frac1L, L]\) and \(x \in [L^*,0].\) It follows directly that \(\lambda x \ge -\ln(2)\) and therefore \(\Delta'_\lambda(x) \ge 0\) for all \(x \in [L^*,0].\) \\
It follows \[\Delta_\lambda(\pi_n)\ge \Delta_\lambda(\pi_n^{cut}).\]
In summary, the regret rate does not get worse by replacing policies by cut-off policies and hence we get 
\begin{align*}
\inf_{\pi \in \Pi} \sup_{\lambda \in \mathcal{H}(L)} \cR^\pi_\lambda(T) \ge \inf_{\pi \in \Pi^{full} } \sup_{\lambda \in \mathcal{H}(L)} \cR^\pi_\lambda(T) =  \inf_{\pi \in \Pi^{cut} } \sup_{\lambda \in \mathcal{H}(L)} \cR^\pi_\lambda(T).    
\end{align*}
Therefore a lower bound for cut-off policies is valid for the policies from Definition \ref{defi policy} as well.\\
When estimating a constant jump intensity, observing \((N_t^j)_{t \in [S,\tau_0^j]}\) is equivalent, in the sense of sufficiency, to observing the pair \((N_0^j, \tau_0^j)\), whose components are independent with \( N_0^j \sim Poi(\lambda |S|)\) and \(\tau_0^j \sim Exp(\lambda)\).  Therefore, in the following we can assume that any estimator  of the parameter \(\lambda\) makes solely use of the values \((N_0^0,\tau_0^0),( N_0^1, \tau_0^1), (N_0^2,\tau_0^2), \dots\). Note that the latter sequence is independent and identically distributed. W.l.o.g. we choose \(|S|=1\). Otherwise just replace \(\lambda\) by \(\tilde{\lambda}:= |S| \lambda \in \left[|S|a, |S| b\right]\) and \(\tau_0^j\) by \(\tau_0^j/|S|\) in the following considerations. 

Recall from Corollary \ref{constant_lambda_cost} that for any constant $\lambda$ we have \(E_\lambda(|\tau_{b^*}|) = \frac{\ln (2)}{\lambda}\).
There is a one-to-one correspondence between the optimal threshold \(b^*\) and the given constant intensity \(\lambda\). Indeed, we have \(b^*(\lambda)= -\frac{\ln (2)}{\lambda}\) and \(\lambda(b^*)=-\frac{\ln (2)}{b^*}.\) We use this correspondence to reduce the problem of finding a minimizing cut-off policy to a problem of determining an estimator for \(\lambda\) with minimal mean squared error. 

\begin{proposition}\label{propo opti gap}
There exists a constant \(c \in (0, \infty)\) such that for all \(x \in [S,0]\) and \(\lambda \in \mathcal{H}(L) \) we have  \[\Delta_\lambda (x) \ge c (x - b^*(\lambda))^2.\]
\end{proposition}
\begin{proof}
We make a Taylor expansion of \(\Delta_\lambda(x)\) around the optimal threshold \(b^*(\lambda)\). Note that all necessary smoothness conditions are fulfilled and therefore we get
\begin{align*}
    \Delta_\lambda(x) = \Delta_\lambda(b^*(\lambda)) + \Delta'_\lambda(b^*(\lambda)) (x- b^*(\lambda)) + \frac12\Delta''_\lambda(\nu) (x- b^*(\lambda))^2,
\end{align*}
where \(\nu \in [\min\{x, b^*(\lambda)\},\max\{x, b^*(\lambda)\}].\) By definition and the first order optimality criteria we know \(\Delta_\lambda(b^*(\lambda))=\Delta'_\lambda(b^*(\lambda))=0.\) Moreover, for \(\lambda \in \mathcal{H}(L)\) we know that \(\Delta_\lambda(x)=\frac{1}{\lambda}(2 e^{\lambda x} -1) -x-\frac{\ln(2)}{\lambda}.\) It follows \(\Delta''_\lambda(x)= 2 \lambda e^{\lambda x} > 0\) for all \(x\). Therefore, we have 
\begin{align*}
    \Delta_\lambda(x) \ge c (x - b^*(\lambda))^2,
\end{align*}
where \(c:=\frac12\inf_{\lambda \in [a,b]} \inf_{x\in [S,0]} \Delta''_\lambda(x) >0.\)
\end{proof}
Proposition \ref{propo opti gap} implies that for all \((\pi_n) \in \Pi^{cut}\) we have 
\begin{align}
\label{correspondence_b_l}
E_\lambda(\Delta_\lambda(\pi_n)) &\ge c E_\lambda[(\pi_n - b^*(\lambda))^2] =  c E_\lambda\left[(b^*(-\frac{\ln (2)}{\pi_n}) - b^*(\lambda))^2\right] \notag \\
&\ge \tilde{c} E_\lambda\left[(-\frac{\ln (2)}{\pi_n} - \lambda)^2\right],
\end{align}
where $\tilde c := c (\min_{x \in [a,b]} (b^*)'(x))^2 = c \frac{\ln(2)^2}{L^4}$. The last inequality follows from the mean value theorem.

We can interpret \(-\frac{\ln (2)}{\pi_n}\) as the corresponding estimator for the parameter \(\lambda\) defined by the policy \((\pi_n) \in \Pi^{cut}\). It follows directly by \eqref{correspondence_b_l} that \[\cR^\pi_\lambda(T) \ge \sum_{n=1}^T E_\lambda (\Delta_\lambda (\pi_n)) \ge \tilde c \sum_{n=1}^T E_\lambda\left[(\frac{-\ln (2)}{\pi_n}-\lambda)^2\right]\]
and therefore 
\[\inf_{\text{policies }\pi} \sup_{\lambda \in \mathcal{H}(L)} \cR^\pi_\lambda(T) \ge \inf_{\pi \in \Pi^{cut} } \sup_{\lambda \in \mathcal{H}(L)} \cR^\pi_\lambda(T)\ge \inf_{\text{estimator }\hat{\lambda}} \sup_{\lambda \in [a,b]} \tilde{c} \sum_{n=1}^T E_\lambda [(\hat{\lambda}_n - \lambda)^2].\]
In summary, we have reduced the problem of deriving a lower bound for the minimax regret on the environment class \(\mathcal{M}(L)\) to a problem of deriving a lower bound for the minimax risk when estimating \(\lambda \in \mathcal{H}(L)\) using the observations \((N_0^0,\tau_0^0),( N_0^1, \tau_0^1), (N_0^2,\tau_0^2), \dots\). 

Next observe that the minimax risk can be bounded from below with the Bayes risk with prior density \(q\)
\begin{align}
\label{eq_beforeVT}
\inf_{\text{estimator }\hat{\lambda}} \ \sup_{\lambda \in [a,b]} E_\lambda \left[ \sum_{n=1}^T(\hat{\lambda}_n - \lambda)^2\right] &\ge \inf_{\text{estimator }\hat{\lambda}} \int_{[a,b]} E_\lambda \left[ \sum_{n=1}^T(\hat{\lambda}_n - \lambda)^2\right] q(\lambda) d\lambda \notag \\
&=\inf_{\text{estimator }\hat{\lambda}} \sum_{n=1}^T \int_{[a,b]} E_\lambda \left[ (\hat{\lambda}_n - \lambda)^2\right] q(\lambda) d\lambda.
\end{align}

Recall that the density of the \(Beta(3,3)-\)distribution on \([0,1]\) is given by \(q_{[0,1]}(x)=30 x^2 (1-x)^2.\) It is straightforward to show that \(q_{[a,b]}(x)=\frac{1}{b-a} q_{[0,1]}(\frac{x-a}{b-a}) = \frac{30}{(b-a)^5} (x-a)^2 (b-x)^2\) is a density function on $[a,b]$. In the following we choose the prior with density $q = q_{[a,b]}$.

Now, we apply the so-called van Trees inequality to the RHS of \eqref{eq_beforeVT}. This inequality can be found, for example, in Chapter 2 in \cite{gill1995applications}. For the choice of the specific prior with density $q = q_{[a,b]}$ all conditions for the application of the van Trees inequality are satisfied, especially the condition \(q(a)=q(b)=0\), and hence we get 
\begin{align*}
\inf_{\text{estimator }\hat{\lambda}} \sum_{n=1}^T \int_{[a,b]} E_\lambda \left[ (\hat{\lambda}_n - \lambda)^2\right] q(\lambda) d\lambda &\ge \sum_{n=1}^T \frac{1}{\mathbb{I}_{q_{[a,b]}} + \int_{[a,b]} \mathbb{I}_{p^{(n)}} (\lambda) q(\lambda) d\lambda},
\end{align*} 
where \(\mathbb{I}_{q_{[a,b]}} := \int \left[ \frac{\partial}{\partial \lambda} \ln(q(\lambda)) \right]^2 q(\lambda) d\lambda\) and \(\mathbb{I}_{p^{(n)}}(\lambda)\) denotes the Fisher information of \((N_0^0,\tau_0^0),( N_0^1, \tau_0^1),  \dots, (N_0^{n-1},\tau_0^{n-1})\). Since the Fisher information is additive for independent observations, and since the Fisher information w.r.t. \(\lambda\) of \(Poi(\lambda)\) equals \(\frac1\lambda\) and that of \(Exp(\lambda)\) equals \(1/\lambda^2\), we have \(\mathbb{I}_{p^{(n)}}(\lambda)=n(\frac1\lambda+\frac{1}{\lambda^2})\).

To derive the value of \(\mathbb{I}_{q_{[a,b]}}\) we first calculate \(\mathbb{I}_{q_{[0,1]}}\). Note that $\mathbb{I}_{q_{[0,1]}} = \int\frac{q'(\lambda)^2}{q(\lambda)} d\lambda$ and hence
\begin{align*}
\mathbb{I}_{q_{[0,1]}}=\int_0^1 \frac{60^2 x^2 (1-x)^2 (1-2x)^2}{30 x^2(1-x)^2} dx = 120 \int_0^1 (1-2x)^2 dx= 40. 
\end{align*}
 
Next observe that \(q_{[a,b]}'(x) = \frac{1}{b-a} q_{[0,1]}'(\frac{x-a}{b-a})\frac{1}{b-a}\) and thus 
\begin{align*}
\mathbb{I}_{q_{[a,b]}} &= \int_a^b \frac{\frac{1}{(b-a)^4}(q_{[0,1]}'(\frac{x-a}{b-a}))^2}{\frac{1}{b-a}q_{[0,1]}(\frac{x-a}{b-a})} dx \\
&=\int_0^1 \frac{1}{(b-a)^3} \frac{(q_{[0,1]}'(u))^2}{q_{[0,1]}(u)} (b-a) du\\
&= \frac{40}{(b-a)^2},
\end{align*}
where we substitute \(u=\frac{x-a}{b-a}.\)
To sum up, we get
\begin{align*}
\sum_{n=1}^T \frac{1}{\mathbb{I}_{q_{[a,b]}} + \int_{[a,b]} \mathbb{I}_{p^{(n)}} (\lambda) q(\lambda) d\lambda} &= \sum_{n=1}^T \frac{1}{\frac{40}{(b-a)^2} + n \int_{[a,b]} (\frac1\lambda+\frac{1}{\lambda^2}) q(\lambda) d\lambda} \\
&\ge \sum_{n=1}^T \frac{1}{\frac{40}{(b-a)^2} + n (\frac1a+\frac{1}{a^2}) \int_{[a,b]} q(\lambda) d\lambda} \\
& \ge C' \ln(T), 
\end{align*}
where  \(C' = \frac{1}{\frac{40}{(b-a)^2} + \frac1a+\frac{1}{a^2}} \). Therefore 
\begin{align*}
\tilde{c} \inf_{\text{estimator }\hat{\lambda}} \sup_{\lambda \in [a,b]} E_\lambda \left[ \sum_{n=1}^T(\hat{\lambda}_n - \lambda)^2\right] \ge \tilde{c} C' \ln (T) =C \ln(T),
\end{align*}
where \(C = \tilde c C'\).
In summary,
\[\inf_{\pi \in \Pi} \sup_{\lambda \in \mathcal{M}(L)} \cR^\pi_\lambda(T) \ge \inf_{\text{estimator }\hat{\lambda}} \sup_{\lambda \in [a,b]} \tilde{c}  E_\lambda \left[ \sum_{n=1}^T(\hat{\lambda}_n - \lambda)^2\right] \ge C \ln(T). \]
Thus Theorem~\ref{main_res_3} is proved.

\newpage
\subsection*{Acknowledgments}
We thank Nabil Kazi-Tani and Michael Neumann for very fruitful discussions. Funded by the Deutsche Forschungsgemeinschaft (DFG, German Research Foundation) – Project number 547236699.

\hspace{-1cm}
\bibliographystyle{abbrv} 
\bibliography{literatur}

\end{document}